\documentclass[10pt,twocolumn,letterpaper]{article}

 \usepackage{cvpr}              
\usepackage[most]{tcolorbox} 
\usepackage{enumitem}
\usepackage{amsthm}
\usepackage{mathtools}
\usepackage{booktabs}
\usepackage{siunitx}
\usepackage{multirow}
\usepackage{graphicx} 
\usepackage{array}      
\usepackage{caption}    
\usepackage{subcaption}
\usepackage{etoc} 
\usepackage{titletoc}
\allowdisplaybreaks
\tcbset{
	Intuitions/.style={
		enhanced,
		sharp corners,
		colback=gray!5,
		colframe=black,
		fonttitle=\bfseries\color{white},
		coltitle=white,
		title filled,
		colbacktitle=black,
		boxrule=0.5pt,
		attach boxed title to top left={yshift=-2mm, xshift=2mm},
		top=1mm, bottom=1mm, left=1mm, right=1mm
	}
}
\usepackage{listings}
\usepackage{xcolor}
\usepackage{marvosym}
\lstdefinestyle{pythonstyle}{
	language=Python,
	basicstyle=\ttfamily\small,
	keywordstyle=\color{blue},
	stringstyle=\color{red},
	commentstyle=\color{gray},
	showstringspaces=false,
	breaklines=true,
}
\newtheorem{theorem}{Theorem}[section]      

\theoremstyle{definition} 

\theoremstyle{remark} 

\definecolor{cvprblue}{rgb}{0.21,0.49,0.74}
\usepackage[pagebackref,breaklinks,colorlinks,allcolors=cvprblue]{hyperref}

\def\paperID{Anonymous} 
\def\confName{CVPR}
\def\confYear{2026}

\title{Closed-Form Concept Erasure via Double Projections}

\author{
	Chi Zhang$^{1}$ \qquad
	Jingpu Cheng$^{1}$ \qquad
	Zhixian Wang$^{2}$ \qquad
	Ping Liu$^{3}$$^{\text{\small\Letter}}$\\
	$^{1}$National University of Singapore \quad
	$^{2}$Shanghai Jiao Tong University \quad
	$^{3}$University of Nevada, Reno\\
	{\tt\small czhang24@nus.edu.sg}
	\quad
	{\tt\small chengjingpu@u.nus.edu}
	\quad
	{\tt\small jd.wzx@sjtu.edu.cn}
	\quad
	{\tt\small pino.pingliu@gmail.com}
}

\begin{document}
\maketitle

\renewcommand{\thefootnote}{}
\footnotetext{{\small\Letter}\, Corresponding author. Accepted to the IEEE/CVF Conference on Computer Vision and Pattern Recognition (CVPR) 2026.}
\renewcommand{\thefootnote}{\arabic{footnote}}

\begin{abstract}
	While modern generative models such as diffusion-based architectures have enabled impressive creative capabilities, they also raise important safety and ethical risks. These concerns have led to growing interest in concept erasure, the process of removing unwanted concepts from model representations. Existing approaches often achieve strong erasure performance but rely on iterative optimization and may inadvertently distort unrelated concepts. In this work, we present a simple yet principled alternative: a linear transformation framework that achieves concept erasure analytically, without any training. Our method adapts a pretrained model through two sequential, closed-form steps: first, computing a proxy projection of the target concept, and second, applying a constrained transformation within the left null space of known concept directions. This design yields a deterministic and geometrically interpretable procedure for safe, efficient, and theory-grounded concept removal. Across a wide range of experiments, including object and style erasure on multiple Stable Diffusion variants and the flow-matching model (FLUX), our approach matches or surpasses the performance of state-of-the-art methods while preserving non-target concepts more faithfully. Requiring only a few seconds to apply, it offers a lightweight and drop-in tool for controlled model editing, advancing the goal of safer and more responsible generative models. Code is available \href{https://github.com/czhang024/CVPR_DouleProjection}{here}.
\end{abstract}

\section{Introduction}
The remarkable capabilities of modern generative models, including diffusion-based~\cite{ho2020denoising,song2020denoising}  and flow-based~\cite{dosovitskiy2020image,lipman2022flow,peebles2023scalable} methods, have revolutionized content creation. These systems can produce diverse, high-fidelity images and text from simple prompts, enabling a wide range of creative and practical applications~\cite{yu2022scaling,rombach2022high,ramesh2022hierarchical_arxiv2022}. Yet this power comes with risks: generative models may inadvertently reproduce copyrighted material, generate biased or harmful content, or reveal sensitive information~\cite{carlini2021extracting,carlini2022quantifying}. Such concerns have made concept erasure~\cite{gandikota2023erasing}, the selective removal of undesired concepts from model representations, an increasingly important direction for safe and responsible AI.


Existing approaches pursue this goal through a range of mechanisms, including cross-attention layer modifications~\cite{gandikota2023erasing,gandikota2024unified,lu2024mace}, model pruning strategies~\cite{yang2024pruningrobustconcepterasing_arxiv2024,chavhan2024conceptprune_iclr2025}, regularization-based editing~\cite{huang2024receler}, and adversarial-guided erasure~\cite{AdvUnlearn_nips2024,bui2025age_iclr2025}. In general, these methods seek to remove target concepts by altering specific model components or parameters such as attention mechanisms, feature representations, or network weights. 
Collectively, these approaches have proven highly effective at suppressing targeted concepts in complex generative models, demonstrating that such information can indeed be localized and selectively removed from internal representations~\cite{gandikota2023erasing}. This progress has also enabled a variety of beneficial applications~\cite{gandikota2024unified}, including removing unwanted objects or artistic styles, enforcing copyright protection, mitigating harmful content, and promoting fairness in generative outputs.

Yet, in doing so, existing methods may also unintentionally affect other, non-target concepts, degrading the model’s overall representational balance. For example, pruning neurons associated with a particular concept~\cite{chavhan2024conceptprune_iclr2025} can also remove neurons critical for other semantic attributes or generative behaviors, leading to noticeable drops in performance.
This raises a central practical question: how can we effectively erase specific concepts while preserving a model’s knowledge of non-target concepts?

To address this challenge, we introduce ``Concept Erasure with Double Projections'' (DP), a principled and efficient framework that explicitly minimizes interference with non-target representations. Instead of relying on iterative optimization or retraining, DP reformulates concept erasure as a pair of analytical projection steps with clear geometric interpretation. The first projection isolates the safe component of a target concept by aligning it with known non-target directions. The second applies a constrained transformation within the left nullspace of preserved representations, ensuring that removing the target concept minimally affects others. Importantly, both steps admit analytically closed-form solutions, yielding a \emph{deterministic, training-free} method that operates in seconds.

We evaluate the proposed method across multiple concept-erasure settings, including object and style erasure, using several variants of Stable Diffusion~\cite{rombach2022high} and the recent flow-matching model FLUX~\cite{lipman2022flow,batifol2025flux}. Across all these architectures, our approach achieves erasure performance comparable to or better than existing state-of-the-art techniques in terms of removing the targeted concepts. More importantly, both qualitative and quantitative results consistently demonstrate that our approach better preserves the remaining non-target concepts, maintaining the overall generative quality and diversity of model outputs. 


Overall, our study demonstrates that concept erasure can be formulated and solved efficiently within a \emph{principled geometric framework}. By decoupling the optimization into two analytically solvable steps, the double projection solution achieves both interpretability and practicality, removing unwanted concepts in seconds without retraining or iterative fine-tuning. In practice, such a projection design offers several key advantages:
\begin{enumerate}[label=(\arabic*)]
	\item \textbf{Closed-form formulation.} We reformulate concept erasure as an analytically solvable linear transformation problem, providing a one-shot solution with provable guarantees and eliminating any need for retraining. 
	\item \textbf{Geometric interpretability.} The proposed double-projection design offers a principled geometric perspective that explicitly characterizes how erasure and preservation interact within the representation space. 
	\item \textbf{Effective erasure and preservation.} Our method achieves state-of-the-art suppression of targeted concepts while minimizing interference with non-target semantics, preserving both visual quality and diversity. 
	\item \textbf{Cross-model generality.} The framework operates consistently across multiple diffusion and flow-matching architectures, demonstrating robustness and scalability for diverse generative backbones. 
\end{enumerate}

\section{Related Works}
\vspace{-0.1in}
\paragraph{Deep Generative Models and Personalization.}  
Deep generative models have become the foundation of modern image synthesis, with diffusion-based and flow-matching architectures leading recent advances \citep{rombach2022high, ramesh2022hierarchical_arxiv2022, ho2020denoising, lipman2022flow, liu2022flow, albergo2023building}.
Diffusion models \citep{ho2020denoising} generate images through iterative denoising from Gaussian noise, guided by learned score functions to produce highly realistic and semantically consistent outputs. Flow-matching methods \citep{lipman2022flow} later introduced deterministic mappings between noise and data distributions, improving sample efficiency by aligning trajectories in a continuous latent space \citep{lipman2022flow, liu2022flow}.
These advances have enabled the synthesis of high-quality, semantically faithful imagery, driving widespread adoption across creative, industrial, and scientific applications. Building on this progress, personalization techniques have been developed to adapt generative models for user-specific concepts from only a few examples. In diffusion models, approaches such as DreamBooth \citep{ruiz2023dreambooth_cvpr2023}, Textual Inversion \citep{gal2022image}, and parameter-efficient tuning \citep{kumari2023multi, shi2024instantbooth} enable subject-driven generation without retraining the full model. More recently, personalization has extended to flow-based architectures, with classifier-guided adaptation \citep{sun2024rectifid} and LoRA-based fine-tuning \citep{dalva2025lorashop} supporting flexible concept encoding and efficient customization.

\vspace{-0.15in}
\paragraph{Risks and Safety Concerns in Deep Generative Models.}
Despite their remarkable versatility, generative models introduce serious ethical and safety challenges. One major concern is copyright infringement: large-scale models trained on web data can memorize and reproduce copyrighted works nearly verbatim \citep{carlini2023extracting, somepalli2023diffusion}, leading to legal disputes with artists and creators \citep{andersen2023lawsuit}. Another is bias amplification—these models often internalize and reinforce stereotypes present in their training data \citep{luccioni2023stable, cho2023dall, bianchi2023easily}, perpetuating harmful associations related to gender, race, or occupation. Generative models are also prone to producing unsafe or explicit content, including violent or pornographic imagery \citep{schramowski2023safe, hunter2023ai}, and safety filters designed to prevent such outputs can often be bypassed \citep{rando2022red}. Furthermore, personalization techniques can exacerbate these issues by enabling malicious use cases such as nonconsensual deepfakes and imitation of artistic styles without consent \citep{shan2023glaze, salman2023raising}.  

\vspace{-0.15in}
\paragraph{Concept Erasure.}  
These multifaceted safety challenges have spurred extensive research on concept erasure techniques aimed at mitigating harmful generative behaviors \citep{gandikota2023erasing, gandikota2024unified, kumari2023ablating, kim2023towards, wang2024unified, chengmachine}. Concept erasure seeks to suppress a model’s ability to reproduce undesired objects, styles, or identities while maintaining generation quality for non-targeted concepts.
Representative approaches include fine-tuning methods~\citep{gandikota2023erasing,heng2023selective}, cross-attention editing~\citep{gandikota2024unified,lu2024mace,huang2024receler}, and attention re-steering~\citep{zhang2024forget}.
Other strategies involve regularization~\citep{zhao2024separable}, pruning~ \citep{chavhan2024conceptprune_iclr2025}, adversarial training~\citep{bui2025age_iclr2025}, Dumo~\citep{han2025dumo}, and trajectory-based techniques~\citep{carter2025trace}. Overall, these methods seek to alter the behavior of pretrained models through post-training modifications, particularly efficient approaches~\cite{hu2022lora,zhang2024parameter,zhang2025weight}. Recent efforts also explore interpretability-driven erasure using sparse autoencoders \citep{cywinski2025saeuron, kim2025concept}, training-free localized erasure via low-rank adaptation \citep{lee2025localized}, and neuron-level precision removal \citep{he2025single}. There is also growing interest in robustness and evaluation~\cite{liu2025erased,chin2023prompting4debugging,zhang2024generate,xie2025erasing}. In addition to these diffusion-based models, extensions on concept erasure have also been proposed for flow-matching models \citep{gao2025eraseanything_icml}, autoregressive transformers \citep{han2025vcesafeautoregressiveimage_arxiv2025}, and text-to-video generation \citep{ye2025t2vunlearningconcepterasingmethod_arxiv2025, xu2025videoeraserconcepterasuretexttovideo_emnlp2025}. Related ideas are explored in large language models through nullspace-based editing \citep{fang2025alphaedit}, which focuses on MLP layers, whereas our work targets attention and embedding layers for visual generative models.


\clearpage
\section{Concept Erasure: Problem Formulation and Geometric Insights}
\subsection{Problem Formulation}
Modern generative models, such as diffusion~\cite{ho2020denoising} and transformer-based architectures~\cite{dosovitskiy2020image}, implicitly encode a rich set of semantic concepts within their latent representations. Let $f_{\theta_0}: \mathbb{R}^n \rightarrow \mathbb{R}^p$ denote the pretrained model parameterized by $\theta_0$, which maps an internal latent code $z \in \mathbb{R}^n$ to an output feature $f_{\theta_0}(z) \in \mathbb{R}^p$. 
Then given a text prompt $c$, the model defines a conditional distribution
\[
p_{\theta_0}(x \mid c), \quad x \in \mathcal{X}, \; c \in \mathcal{C},
\]
representing the likelihood of generating an image $x$ conditioned on the prompt $c$.

Let $\mathcal{C}_{\text{target}} \subseteq \mathcal{C}$ denote the set of prompts corresponding to target concepts (e.g., specific objects, styles, or identities) to be erased, and let $\mathcal{X}_{\text{target}} \subseteq \mathcal{X}$ denote the associated undesired outputs.
The goal of concept erasure is to transform the model parameters from $\theta_0$ to $\theta$ such that the target concepts are effectively suppressed, while preserving the model’s ability to generate and represent a set of non-target concepts. Formally, we seek a transformation $\{\theta_0 \rightarrow \theta\}$ satisfying two complementary objectives: \emph{erasure} and \emph{preservation}.

\vspace{-0.15in}
\paragraph{Erasure objective.}
The primary requirement is that the modified model should not produce undesired content when conditioned on any target prompt. This can be expressed as
\begin{equation}
	\forall \, c \in \mathcal{C}_{\text{target}}:\quad 
	\mathrm{supp}\!\big(p_\theta(x \mid c)\big)\;\cap\;\mathcal{X}_{\text{target}} \;=\; \emptyset,
	\label{eq:hard-erasure}
\end{equation}
where $\mathrm{supp}(p_\theta(x \mid c))$ denotes the \emph{support} of the conditional distribution~\cite{folland1999real}, i.e., the set of all possible samples $x$ that the model can produce with nonzero probability under prompt $c$. In practice, this strict condition is relaxed to a probabilistic form:
\begin{equation}
	\forall \, c \in \mathcal{C}_{\text{target}}:\quad 
	\mathbb{P}_{x \sim p_\theta(\cdot \mid c)}\!\left[x \in \mathcal{X}_{\text{target}}\right] \;\le\; \delta,
	\label{eq:soft-erasure}
\end{equation}
where $\delta \ge 0$ sets a tolerance for the probability of undesired content generations.

\vspace{-0.15in}
\paragraph{Preservation objective.}
Equally important is preserving the model’s capabilities on non-target prompts. Ideally, the modified model should exhibit identical behavior to the original model on all non-target concepts. 
Let $\mathcal{C}_{\text{pres}} \subseteq \mathcal{C}$ denote the set of prompts to be preserved, and let $p_{\theta_0}$ denote the pretrained model’s distribution. 
A natural formulation is to require that the generated distributions remain close under some divergence measure $D(\cdot\Vert\cdot)$:
\begin{equation}
	\forall \, c \in \mathcal{C}_{\text{pres}}:\quad 
	D\ \!\big(p_\theta(\cdot \mid c)\, \Vert \, p_{\theta_0}(\cdot \mid c)\big) \;\le\; \varepsilon,
	\label{eq:preservation}
\end{equation}
where $\varepsilon \ge 0$ controls the tolerance for deviation. Preservation can be expressed through alignment in feature space using a functional $\phi(\cdot)$ (e.g., CLIP embeddings~\cite{radford2021learning}, perceptual features, or aesthetic scores):
\begin{equation}
	\forall c \in \mathcal{C}_{\text{pres}}:\;
	\|\mathbb{E}_{x \sim p_\theta(\cdot \mid c)}[\phi(x)]
	-\mathbb{E}_{x \sim p_{\theta_0}(\cdot \mid c)}[\phi(x)]\|_2
	\le \varepsilon_\phi.
	\label{eq:functional-preservation}
\end{equation}

Empirical measures of this divergence include Maximum Mean Discrepancy (MMD), Fréchet Inception Distance (FID)~\cite{heusel2017gans}, or performance-based metrics such as classification accuracy.

\subsection{Empirical Concept Erasure with UCE}
Directly optimizing the objectives in Eqs.~\eqref{eq:soft-erasure}, \eqref{eq:preservation}, and \eqref{eq:functional-preservation} within the full parameter space of a generative model is typically infeasible, owing to the distributed nature of concept representations, the high dimensionality of model parameters, and the nonlinear behavior of modern architectures. For practical deployment, empirical approaches operate within a restricted subspace, often targeting specific projection layers or attention matrices that encode concept-level information.

A representative example is Unified Concept Editing (UCE)~\cite{gandikota2024unified}, which applies modifications to selected model components (e.g., the Key and Value matrices in attention layers) and formulates the empirical objective as:
\begin{equation}
	\min_{W} 
	\Biggl(
	\sum_{c_i \in \mathcal{C}_{\text{target}}} 
	\| W c_i - W_0 c_i^* \|_2^2
	\;+\;
	\sum_{c_j \in \mathcal{C}_{\text{pres}}} 
	\| W c_j - W_0 c_j \|_2^2
	\Biggr),
	\label{eq:empirical-formulation}
\end{equation}
where $c_i^*$ denotes a proxy representation of the erased concept, typically chosen as a neutral anchor. The empirical objective in~\eqref{eq:empirical-formulation} consists of two complementary parts: the first term enforces the erasure of target concepts by aligning their transformed representations $W c_i$ with neutral proxy $W_0 c_i^*$, while the second term preserves non-target concepts by constraining the new mapping $W$ to remain close to $W_0$ on preserved prompts.

A notable advantage of UCE is that it admits a \emph{closed-form solution}, allowing direct computation of the optimal projection matrix without iterative training:
\begin{equation}
	\scriptsize
	W = 
	\Bigl(\!\!\sum_{c_i \in \mathcal{C}_{\text{target}}} v_i^* c_i^{\!\top}
	+\!\!\sum_{c_j \in \mathcal{C}_{\text{pres}}} W_0 c_j c_j^{\!\top}\Bigr)
	\Bigl(\!\!\sum_{c_i \in \mathcal{C}_{\text{target}}} c_i c_i^{\!\top}
	+\!\!\sum_{c_j \in \mathcal{C}_{\text{pres}}} c_j c_j^{\!\top}\Bigr)^{-1},
	\label{eq:uce-solution}
\end{equation}
where \(v_i^*=W_0 c_i^*\) is the desired target vector.  

This closed-form solution offers several practical advantages. 
First, it enables one-step computation of the updated projection matrix, avoiding iterative gradient-based optimization or retraining, which substantially reduces computational overhead. 
Moreover, by operating solely on the concept embeddings $\mathcal{C}_{\text{target}}$ and $\mathcal{C}_{\text{pres}}$, the method is entirely data-independent and does not require additional image sampling or backpropagation through the generative model. 

\subsection{Closed-Form Solution $\neq$ Good Solution}
While closed-form approaches such as UCE offer clear advantages in efficiency, they do not inherently guarantee the preservation of non-target concepts. Despite the inclusion of a preservation term in Eq.~\eqref{eq:empirical-formulation}, violations on preserved prompts can still occur—particularly when the target and preserved concepts are correlated or non-orthogonal in the latent space. The following geometric insights provide an intuitive understanding of why violations of preserved concepts may still occur.
\begin{tcolorbox}[
	title={Geometric Insights},
	colback=gray!5,
	colframe=gray!40!black,
	boxrule=0.4pt,
	arc=2pt,
	left=3pt,
	right=3pt,
	top=3pt,
	bottom=3pt
	]
	In least-squares regression, the best-fit line minimizes total error but does not necessarily pass through every data point. Concept erasure in Eq.~\eqref{eq:empirical-formulation} behaves similarly: each concept embedding is a point in a high-dimensional space, and the optimization only finds a transformation $W$ that minimizes the global loss. As such, target and preserved concepts may not lie on the fitted line, or even deviate significantly from the fitted solution, causing degradation and distortions in these concepts.
\end{tcolorbox}

The following theorem provides a more concrete analysis of this phenomenon.


\begin{theorem}[Perturbation of Preserved Concepts]
	\label{thm:perturbations}
	Assume there is only one target vector $c$ to be edited to $v^*$ and let $C_{\mathrm{pres}}$ denote the concatenated preservation matrix. Let $N=c c^\top+C_{\mathrm{pres}}C_{\mathrm{pres}}^\top$, and assume that for some preserve vector $p$, $\langle N^{-1/2}c, N^{-1/2}p\rangle \ge \lambda \langle N^{-1/2}c, N^{-1/2}c\rangle$ for some $\lambda >0$. Then we have 
	\begin{equation}
		\|\Delta W p\|_2 \ge \lambda \|\Delta W c\|_2.
	\end{equation}
	That is, the perturbation on the non-target vector $p$ is at least $\lambda$ times the perturbation on the target vector $c$.
\end{theorem}

In practice, we also observe this phenomenon consistently across different models and concept sets. 
Non-target concepts experience noticeable degradation for both object and style erasure in Table~\ref{tab:sd14_objects} and \ref{tab:sd14_style}. These observations underscore a crucial limitation: achieving a mathematically optimal solution under a least-squares objective does not imply controlled erasure and preservation. Maintaining their performance instead requires a more deliberate geometric design.

\section{Concept Erasure with Double Projections}
The geometric limitations of existing closed-form approaches motivate a more principled formulation of concept erasure. To this end, we propose ``Concept Erasure with Double Projections'' (DP), which explicitly decouples erasure and preservation through two sequential projections. By disentangling subspace interactions, DP provides analytical guarantees for training-free updates while retaining the efficiency and interpretability of a closed-form solution.

\subsection{Formulation}
We aim to identify an updated transformation $W \in \mathbb{R}^{p\times n}$ that effectively removes the representations of specific target concepts while preserving those of non-target concepts. 
Formally, we write $W = W_0 + \Delta W$ and optimize directly over $\Delta W$:
\begin{equation} 
	\min_{W \in \mathbb{R}^{p\times n},\, c_i^* \in \mathcal{S}} \left( \| W c_i - W_0 c_i^* \|_2^2 + \| W C_{\text{pres}} - W_0 C_{\text{pres}} \|_F^2 \right), 
	\label{eq:double-proj-objective} 
\end{equation}
where $W_0 \!\in\! \mathbb{R}^{p\times n}$ is the pretrained parameter matrix (e.g., an attention Key or Value matrix),  
$c_i \!\in\! \mathbb{R}^n$ is the embedding of a target concept to be erased,  
$C_{\text{pres}} = [c_1, c_2, \dots, c_m] \!\in\! \mathbb{R}^{n\times m}$ collects the embeddings of preserved (non-target) concepts,  
and $\mathcal{S}$ defines the safe subspace within which the proxy vectors $c_i^*$ are constrained to lie.

In essence, this optimization problem~\eqref{eq:double-proj-objective} involves two sets of variables: the weight matrix $W$ and the proxy vector $c_i^*$. A common approach to solving such problems is through ``alternating optimization''~\cite{boyd2011distributed,lee2000algorithms,bezdek2003convergence}, which iteratively updates one variable while keeping the other fixed until convergence. Yet, these iterative procedures typically rely on gradient-based training and can be \emph{computationally expensive}. Instead, we introduce a novel double projection method that yields a \emph{closed-form, training-free} approximation to this optimization problem.

\subsection{Projection 1: Proxy Construction in the Safe Subspace}
We begin by computing a proxy vector $c_i^*$ that captures the component of the target concept $c_i$ lying within the safe subspace $\mathcal{S}$. 
Let $S \in \mathbb{R}^{n \times k}$ denote the matrix whose columns form a (possibly non-orthogonal) basis of $\mathcal{S}$. 
The proxy is then obtained through an orthogonal projection:
\begin{equation}
	c_i^* = \mathrm{proj}_{\mathcal{S}}(c_i) = S(S^\top S)^+ S^\top c_i.
	\label{eq:proxy-vector}
\end{equation}
where $(S^\top S)^+$ denotes the general Moore--Penrose pseudoinverse~\cite{moore1920reciprocal}, ensuring the projection remains valid even if the basis vectors are linearly dependent. 
This step extracts the \emph{safe component} of the target concept within the non-target subspace, effectively filtering out directions that could interfere with preserved concepts. 
In practice, one can construct a safe region by using multiple safe concepts, $\mathcal{S} = \mathrm{span}\{s_1, s_2, \dots, s_k\}$. 
Note that when $\mathcal{S}$ is defined using a single concept vector ($k=1$), we require Eq.~\eqref{eq:proxy-vector} to collapse to the UCE case~\cite{gandikota2024unified}.

\subsection{Projection 2: Constrained Optimization for $W$}
Given the proxy $c_i^*$ from Projection~1, we now optimize the transformation $\Delta W$ while guaranteeing that updates are \emph{orthogonal} to the space of preserved concepts. Let the preserved (non-target) concept embeddings be collected as 
\[
C_{\text{pres}} = [\, c_1^{\text{pres}},\, c_2^{\text{pres}},\, \dots,\, c_m^{\text{pres}} \,] \in \mathbb{R}^{n\times m},
\]
whose column space defines the subspace that must remain invariant during erasure. 
We parametrize the updated transformation as
\[
\vspace{-0.02in}
W \;=\; W_0 + \Delta W, 
\qquad \text{s.t.}\quad \Delta W\, C_{\text{pres}} = 0,
\vspace{-0.02in}
\]
so that any change lies in the left nullspace of $C_{\text{pres}}$ and therefore leaves the preserved concepts untouched. 

Substituting this into Eq.~\eqref{eq:double-proj-objective} reduces the problem to
\begin{equation}
	\min_{\Delta W}\; \big\|(W_0 + \Delta W)\,c_i - W_0 c_i^*\big\|_2^2
	\quad \text{s.t.}\ \Delta W\, C_{\text{pres}} = 0,
	\label{eq:constrained-optimization}
\end{equation}
a linearly constrained least-squares problem in $\Delta W$.

Let $U_2 \in \mathbb{R}^{n\times (n-r)}$ be an orthonormal basis for the left nullspace of $C_{\text{pres}}$ (\,$r=\mathrm{rank}(C_{\text{pres}})$\,), so that $U_2^\top C_{\text{pres}}=0$. Any feasible update can be written as $\Delta W = Z U_2^\top$ with an unknown parameter $Z \in \mathbb{R}^{p\times (n-r)}$. Define
\[
x \;=\; U_2^\top c_i \in \mathbb{R}^{n-r},
\quad
b \;=\; W_0(c_i^* - c_i) \in \mathbb{R}^{p}.
\]
Eq.~\eqref{eq:constrained-optimization} becomes a standard least-squares problem,
\[	
\vspace{-0.05in}
\min_{Z}\; \big\| Z x - b \big\|_2^2,
\]
whose minimum-norm solution (when $x\neq 0$) is
\begin{equation}
	Z^\star \;=\; b\, x^\top \big(x x^\top\big)^{+} \;=\; \frac{b\, x^\top}{\|x\|_2^2}.
	\label{eq:z-solution}
\end{equation}
The update is therefore admitting a \emph{closed-form solution}:
\begin{equation}
	\Delta W^\star \;=\; Z^\star U_2^\top \;=\; \frac{W_0(c_i^* - c_i)\, x^\top}{\|x\|_2^2}\, U_2^\top.
	\label{eq:final-w-solution}
\end{equation}

Note for multiple-concept erasure  $C_{\mathrm{tgt}}=[c_1,\ldots,c_T]$, we can solve \eqref{eq:z-solution} for concept matrix $X$ in a similar way.

\subsection{Discussions and Geometric Insights}
The first projection is \emph{optional}, and one could directly specify a proxy $c_i^*$ as in UCE~\cite{gandikota2024unified} for simplicity, effectively bypassing this projection. However, constructing a richer safe subspace $\mathcal{S}$ generally reduces the magnitude of the update, leading to smaller $\|\Delta W\|_F^2$ and thus less perturbations to the original model. In contrast, the second projection is \emph{essential}: constraining $\Delta W$ to the nullspace of $C_{\text{pres}}$ guarantees orthogonality, ensuring that model modifications minimally affect the preserved representations. This geometric intuition is made precise in the following theorem.

\begin{theorem}[Preservation of Non-Target Concepts]
	\label{thm:orth-pres}
	Let \(C_{\mathrm{pres}} \in \mathbb{R}^{n\times m}\) denote the matrix of non-target concept embeddings, and let \(W_0 \in \mathbb{R}^{p\times n}\) be the pretrained transformation. 
	If the update \(\Delta W \in \mathbb{R}^{p\times n}\) satisfies \(\Delta W\, C_{\mathrm{pres}} = 0\), then for \(W^\star \coloneqq W_0 + \Delta W\) it holds that \(W^\star v = W_0 v\) for all \(v \in \mathrm{col}(C_{\mathrm{pres}})\); that is, all non-target concept representations are exactly preserved.
\end{theorem}


Most importantly, both projections in Eqs.~\eqref{eq:proxy-vector} and~\eqref{eq:final-w-solution} admit exact \emph{closed-form solutions}. Each step, from computing the proxy vector $c_i^*$ to updating the transformation $W$, can be derived analytically without any iterative optimization or gradient-based training. This makes the entire procedure fully \emph{deterministic and training-free}, combining computational efficiency with clear geometric interpretability. In practice, $c_i^*$ and $C_{\mathrm{pres}}$ are shared by all layers. Moreover, these closed-form updates enable DP to be performed within seconds, in contrast to optimization-based approaches that often require minutes or hours. 

Meanwhile, it is not necessary to include all available concepts as preservation targets. Studies from AGE~\cite{bui2025age_iclr2025} indicate that concept erasure exhibits a largely \emph{localized} effect: removing one concept mainly affects a small neighborhood of semantically related concepts, which can be identified through a concept graph. Hence, $C_{\mathrm{pres}}$ can be constructed from a compact, semantically relevant subset. Moreover, when $C_{\mathrm{pres}}$ includes many concepts, a low-rank truncation can be applied via its singular value decomposition,
$
C_{\mathrm{pres}} = U_1 \Sigma V^\top,
$
where $U_1 = [u_1, \ldots, u_r] \in \mathbb{R}^{n \times r}$ contains left singular vectors ordered by singular values 
$\sigma_1 \ge \cdots \ge \sigma_r > 0$. 
Retaining only the top-$k$ components,
$
U_{1,k} = [u_1, \ldots, u_k], 
$
captures the dominant subspace while discarding low-energy, redundant directions. 
The update rule in Eq.~\eqref{eq:constrained-optimization} can then be parameterized as 
$\Delta W = Z\,U_{2,k}^\top$, 
admitting a similar closed-form solution. The following theorems provide the lower bound for the erasing targets and the upper bound for non-target concepts.

\begin{table*}[thpb]
	\centering
	\caption{Results on SD~1.4 for all algorithms. Each block includes both original and post-update accuracies. 
		Left: \textbf{Target Class} shows erasure performance (\textbf{Erased Accuracy}~$\downarrow$). 
		Right: \textbf{Other Classes} reports the accuracy of preserved concepts (\textbf{Preservation Drop}~$\downarrow$). Lower values indicate stronger erasure and better preservation. \textit{Note:} UCE and DP are \textbf{deterministic methods} and thus no standard deviation is reported. $^{\dagger}$Detailed analysis of why UCE performs worse in these two cases are provided in the Appendix~D.}
	\label{tab:sd14_objects}
	\vspace{-0.12in}
	\resizebox{0.98\textwidth}{!}{
		\begin{tabular}{l|c|ccccc|c|ccccc}
			\toprule
			\multirow{2}{*}{\textbf{Object}} &
			\multicolumn{1}{c|}{\textbf{Target Class}} &
			\multicolumn{5}{c|}{\textbf{Erased Accuracy (\%)~$\downarrow$}} &
			\multicolumn{1}{c|}{\textbf{Other Classes}} &
			\multicolumn{5}{c}{\textbf{Preservation Drop (\%)~$\downarrow$}} \\
			\cmidrule(lr){2-2}\cmidrule(lr){3-7}\cmidrule(lr){8-8}\cmidrule(lr){9-13}
			& Original  & ESD & CP & AGE & UCE & DP & Original & ESD & CP & AGE & UCE & DP \\
			\midrule
			Cassette Player   & $78.0$ & $20.5_{\pm1.5}$ &  $4.3_{\pm0.3}$ & $24.0_{\pm1.0}$ & $12.0^{\dagger}$ & $\textbf{2.0}$ & $86.8$ & $22.6_{\pm3.1}$ & $31.7_{\pm3.5}$ &  $5.8_{\pm0.4}$ & $20.8$ & $\textbf{3.3}$ \\
			Chain Saw         & $80.0$ &  $\textbf{0.0}_{\pm0.0}$ &  $1.0_{\pm0.0}$ &  $1.0_{\pm0.0}$ &  $\textbf{0.0}$ & $\textbf{0.0}$ & $86.6$ & $22.4_{\pm1.7}$ & $34.1_{\pm3.2}$ &  $6.7_{\pm0.5}$ &  $0.6$ & $\textbf{0.3}$ \\
			Church            & $84.0$ &  $\textbf{4.0}_{\pm1.0}$ & $24.6_{\pm0.6}$ & $16.3_{\pm1.7}$ &  $\textbf{4.0}$ & $\textbf{4.0}$ & $86.1$ & $25.3_{\pm2.1}$ & $17.4_{\pm1.0}$ &  $7.2_{\pm0.4}$ & $15.1$ & $\textbf{6.1}$ \\
			Gas Pump          & $77.0$ &  $8.3_{\pm1.0}$ &  $4.0_{\pm0.7}$ &  $4.0_{\pm1.0}$ &  $6.0$ & $\textbf{2.0}$ & $86.9$ & $14.1_{\pm1.7}$ & $31.8_{\pm2.1}$ &  $3.9_{\pm0.3}$ &  $5.6$ & $\textbf{2.6}$ \\
			Tench             & $73.0$ &  $3.0_{\pm1.0}$ &  $\textbf{0.0}_{\pm0.0}$ & $12.0_{\pm2.0}$ &  $\textbf{0.0}$ & $\textbf{0.0}$ & $87.3$ & $12.1_{\pm1.1}$ & $34.3_{\pm3.1}$ &  $6.4_{\pm0.3}$ &  $9.9$ & $\textbf{4.6}$ \\
			Garbage Truck     & $78.0$ & $20.5_{\pm3.0}$ & $\textbf{0.0}_{\pm0.0}$ & $21.0_{\pm2.7}$ &  $\textbf{0.0}$ & $\textbf{0.0}$ & $86.8$ & $11.3_{\pm0.9}$ & $37.2_{\pm4.2}$ &  $4.6_{\pm0.3}$ &  $0.1$ & $\textbf{-1.2}$ \\
			English Springer  & $95.0$ &  $4.7_{\pm0.7}$ & $\textbf{0.0}_{\pm0.0}$ &  $\textbf{0.0}_{\pm0.0}$ &  $\textbf{0.0}$ & $\textbf{0.0}$ & $84.9$ & $20.7_{\pm1.2}$ & $30.9_{\pm3.3}$ &  $5.3_{\pm0.5}$ &  $0.0$ & $\textbf{-0.8}$ \\
			Golf Ball         & $99.0$ &  $6.3_{\pm1.0}$ & $21.3_{\pm1.3}$ &  $6.0_{\pm1.0}$ & $56.0^{\dagger}$ & $\textbf{0.0}$ & $84.4$ & $28.2_{\pm0.0}$ & $32.5_{\pm0.0}$ &  $5.6_{\pm0.0}$ &  $6.2$ & $\textbf{5.4}$ \\
			Parachute         & $95.0$ &  $4.0_{\pm0.3}$ &  $1.0_{\pm0.0}$ & $12.0_{\pm1.3}$ &  $\textbf{0.0}$ & $\textbf{0.0}$ & $84.9$ & $17.4_{\pm2.1}$ & $39.9_{\pm4.1}$ &  $2.6_{\pm0.3}$ &  $4.0$ & $\textbf{-0.8}$ \\
			French Horn       & $100.0$ &  $1.0_{\pm0.0}$ &  $1.0_{\pm0.0}$ &  $\textbf{0.0}_{\pm0.0}$ &  $\textbf{0.0}$ & $\textbf{0.0}$ & $84.3$ & $20.4_{\pm2.1}$ & $34.7_{\pm3.1}$ &  $7.9_{\pm0.6}$ &  $4.7$ & $\textbf{4.2}$ \\
			\midrule
			\textbf{Mean}     & $85.9$ &  $7.2$ &  $5.7$ &  $9.6$ &  $7.8$ & $\textbf{0.8}$ & $85.9$ & $19.5$ & $32.5$ &  $5.6$ &  $6.7$ & $\textbf{2.4}$ \\
			\bottomrule
		\end{tabular}}
	\vspace{-0.15in}
\end{table*}

\begin{theorem}[Preservation Bound for Truncated Cases]
	\label{thm:preservation_bound}
	For any preserve vector $p_i \in \mathcal C_{\mathrm{pres}}$, we have
	\begin{equation}
		\label{eq:preservation_bound}
		\|(W'-W_0)p_i\|_2\le  \|Z^*\|_2\sigma_{k+1}(C_{\mathrm{pres}}),
	\end{equation}
	where $\sigma_{k+1}$ denotes the $k+1$-th singular value of $C_{\mathrm{pres}}$.
\end{theorem}

Note that when $k = r$ (i.e., without truncation), the right-hand side of \eqref{eq:preservation_bound} vanishes, degenerating to Theorem~\ref{thm:orth-pres}. Moreover, let $C_{\mathrm{tgt}}^*$ be the anchor concepts and $C_\perp^{(k)}:=U_{2,k}^\top C_{\mathrm{tgt}}$. Define $B:=W_0\big(C_{\mathrm{tgt}}^\star-C_{\mathrm{tgt}}\big)$ (See Appendix~A.1 for detailed explanations on notations), and we have the following erasure bound.

\begin{theorem}[Erasure Bound for Truncated Cases]
	\label{thm:erasure_bound}
	Let the thin SVD of $C_\perp^{(k)}$ be
	$C_\perp^{(k)}=U_{q_k}^{(k)}\Sigma_{q_k}^{(k)} V_{q_k}^{(k)\top}$ with rank
	$q_k\in\{0,\ldots,\min(n-k,T)\}$.
	We have
	\begin{equation}
		(W-W_0)C_{\mathrm{tgt}} = B\,V_{q_k}^{(k)} V_{q_k}^{(k)\top}.
	\end{equation}
	Moreover, assume that $\ker B$ and the row space of $C_\perp^{(k)}$ intersect trivially, i.e.,
	$\ker B \cap \mathrm{row}(C_\perp^{(k)}) = \{0\}$.
	Equivalently, $B V_{q_k}^{(k)}$ has full column rank and $\sigma_{\min}(B V_{q_k}^{(k)})>0$.
	Then, for each target column $c_i$, letting $y_i:=V_{q_k}^{(k)\top} e_i$,
	\begin{equation}
		\|(W-W_0)c_i\|_2 \;\ge\; \sigma_{\min}\big(BV_{q_k}^{(k)}\big)\,\|y_i\|_2.
	\end{equation}
\end{theorem}

\section{Experiments}
We now turn to the empirical evaluation of concept erasure and preservation, examining the proposed approach under various scenarios such as object and style removal across different Stable Diffusion variants and modern flow-matching models.

\subsection{Experimental Setup}
\paragraph{Backbones and Tasks} Our experiments are first conducted on Stable Diffusion v1.4 (SD1.4), the most widely used backbone in prior concept-erasure studies. To assess generality across architectures, we further evaluate our method on Stable Diffusion v1.5 (SD1.5) and the recent flow-matching generative model FLUX~\cite{lipman2022flow,batifol2025flux}. 
Following prior work~\cite{gandikota2023erasing,gandikota2024unified}, we consider two standard evaluation tracks: (i)~object-level erasure on ten ImageNet categories including \emph{cassette player}, \emph{chain saw}, \emph{church}, \emph{gas pump}, \emph{tench}, \emph{garbage truck}, \emph{English springer}, \emph{golf ball}, \emph{parachute}, and \emph{French horn} and (ii)~style-level erasure targeting five artistic concepts including \emph{Pablo Picasso}, \emph{Vincent van Gogh}, \emph{Rembrandt}, \emph{Andy Warhol}, and \emph{Caravaggio}. For object-level erasure, we report the Top-1 classification accuracy of a pretrained ResNet-50~\cite{he2016deep} on generated images. 
For style-level erasure, we measure the CLIP~\cite{radford2021learning} text–image similarity between generated samples and the corresponding style prompts as ~\cite{gandikota2023erasing,gandikota2024unified,bui2025age_iclr2025}.

\vspace{-0.2in}
\paragraph{Erasure Methods}
We benchmark representative concept-erasure methods that collectively span projection-, fine-tuning-, adversarial-, and pruning-based paradigms. 
Specifically, we compare against Unified Concept Editing (UCE)~\cite{gandikota2024unified}, Erased Stable Diffusion (ESD)~\cite{gandikota2023erasing},  ConceptPrune (CP)~\cite{chavhan2024conceptprune_iclr2025} and AGE~\cite{bui2025age_iclr2025}. These baselines cover a diverse methodological spectrum, enabling a comprehensive assessment of DP’s effectiveness and efficiency relative to existing approaches. For each method, ten image variants are generated per prompt.

\subsection{Object Erasure with Stable Diffusion}
We first focus on object-level concept erasure using Stable Diffusion v1.4 (SD 1.4), a canonical benchmark backbone for prior erasure studies~\cite{gandikota2023erasing,gandikota2024unified,bui2025age_iclr2025}. We follow these prior works in selecting the same ten ImageNet object categories to ensure comparability with established erasure benchmarks. However, our evaluation protocol adopts a \textbf{stricter and more realistic criterion} than previous studies. Specifically, we unify visually and semantically similar concepts (e.g., treating ``cassette player'' and ``tape player'' as equivalent categories) to mitigate the category ambiguity in diffusion outputs. Furthermore, unlike prior evaluations~\cite{bui2025age_iclr2025} that relied on Top-5 accuracy, we report Top-1 accuracy throughout.

Table~\ref{tab:sd14_objects} reports the performance of all algorithms. In terms of \textbf{concept erasure}, several existing methods achieve strong suppression of the target object, confirming that diffusion backbones are generally amenable to concept-level editing.
Methods like UCE, CP and DP, for instance, demonstrate effective removal on easily separable categories such as ``Chain Saw'' and ``English Springer'', where the erased accuracy drops close to zero. These results indicate that when the concept subspace is well localized, single-projection or pruning-based updates can adequately diminish target activations.

\begin{table*}[h]
	\centering
	\caption{Results on SD~1.4 for \textbf{artistic style erasure}. Each block includes both original and post-update accuracies. \textit{Note:} each artist is given a set of its own labels to compute the CLIP score. UCE and DP are \textbf{deterministic methods} and no standard deviation is reported.}
	\label{tab:sd14_style}
	\vspace{-0.15in}
	\resizebox{0.98\textwidth}{!}{
		\begin{tabular}{l|c|ccccc|c|ccccc}
			\toprule
			\multirow{2}{*}{\textbf{Style}} &
			\multicolumn{1}{c|}{\textbf{Target Class}} &
			\multicolumn{5}{c|}{\textbf{Erased Accuracy (\%)~$\downarrow$}} &
			\multicolumn{1}{c|}{\textbf{Other Classes}} &
			\multicolumn{5}{c}{\textbf{Preservation Drop (\%)~$\downarrow$}} \\
			\cmidrule(lr){2-2}\cmidrule(lr){3-7}\cmidrule(lr){8-8}\cmidrule(lr){9-13}
			& Original & ESD & CP & AGE & UCE & DP & Original & ESD & CP & AGE & UCE & DP \\
			\midrule
			Andy Warhol        & $86.0$ & $14.7_{\pm1.3}$ & $21.5_{\pm2.0}$ & $31.5_{\pm3.0}$ & $15.0$ & $\textbf{12.5}$ & $92.1$ &  $4.9_{\pm2.1}$ & $10.4_{\pm1.4}$ & $5.2_{\pm0.7}$ & $2.1$ & $\textbf{2.0}$ \\
			Caravaggio         & $82.0$ & $23.3_{\pm1.3}$ & $20.3_{\pm1.3}$ & $\textbf{7.5}_{\pm1.0}$ & $16.0$ & $11.5$ & $93.6$ &  $8.8_{\pm2.1}$ & $27.5_{\pm3.2}$ & $11.4_{\pm1.1}$ & $0.7$ & $\textbf{0.4}$ \\
			Pablo Picasso      & $81.5$ & $40.5_{\pm2.0}$ & $\textbf{24.5}_{\pm2.5}$ & $26.0_{\pm1.0}$ & $30.0$ & $26.0$ & $91.4$ & $14.7_{\pm1.2}$ & $13.4_{\pm2.1}$ & $4.3_{\pm0.3}$ & $2.9$ & $\textbf{2.6}$ \\
			Rembrandt          & $85.0$ & $18.0_{\pm2.0}$ & $15.0_{\pm1.0}$ &  $3.5_{\pm0.5}$ &  $4.5$ & $\textbf{3.0}$ & $84.4$ &  $8.8_{\pm0.9}$ & $23.3_{\pm2.1}$ & $6.9_{\pm0.7}$ & $-1.0$ & $\textbf{-1.4}$ \\
			Van Gogh           & $63.0$ &  $8.5_{\pm0.5}$ & $14.7_{\pm1.7}$ & $17.7_{\pm0.3}$ &  $7.0$ & $\textbf{5.5}$ & $89.9$ &  $6.7_{\pm0.6}$ &  $8.9_{\pm1.2}$ & $5.8_{\pm0.4}$ & $0.8$ & $\textbf{-1.3}$ \\
			\midrule
			\textbf{Mean}      & $79.5$ & $21.0$ & $19.2$ & $17.2$ & $14.5$ & $\textbf{11.7}$ & $90.3$ & $8.8$ & $16.7$ & $6.7$ & $1.1$ & $\textbf{0.5}$ \\
			\bottomrule
		\end{tabular}}
	\vspace{-0.05in}
\end{table*}

\begin{table*}[t]
	\centering
	\caption{Results on the FLUX model for object erasure. 
		Rows 1--3 report target erasure (\textbf{Erased Accuracy}~$\downarrow$). 
		Rows 4--6 report preservation fidelity (\textbf{Preservation Drop}~$\downarrow$). 
		Visualization of sample generated images is available in Appendix~\ref{sec:vis_flux}.}
	\vspace{-0.1in}
	\label{tab:flux_obj}
	\resizebox{0.98\linewidth}{!}{
		\begin{tabular}{lccccccccccc}
			\toprule
			\textbf{Metric} & Cassette Player & Chain Saw & Church & Gas Pump & Tench & Garbage Truck & English Springer & Golf Ball & Parachute & French Horn & \textbf{Mean} \\
			\midrule
			\multicolumn{12}{c}{\textbf{Target Class (Erased Accuracy, \% $\downarrow$)}} \\
			Original & 39.0 & 100.0 & 99.0 & 100.0 & 89.0 & 98.0 & 82.0 & 100.0 & 99.0 & 100.0 & 90.6 \\
			UCE & \textbf{0.0} & \textbf{0.0} & 63.0 & 73.0 & 6.0 & 21.0 & \textbf{0.0} & 76.0 & \textbf{0.0} & \textbf{0.0} & 23.9 \\
			DP & \textbf{0.0} & \textbf{0.0} & \textbf{12.0} & \textbf{0.0} & \textbf{0.0} & \textbf{0.0} & \textbf{0.0} & \textbf{1.0} & \textbf{0.0} & \textbf{0.0} & \textbf{1.0} \\
			\midrule
			\multicolumn{12}{c}{\textbf{Other Classes (Preservation Drop, \% $\downarrow$)}} \\
			Original & 96.3 & 89.6 & 89.7 & 89.6 & 90.8 & 89.8 & 91.6 & 89.6 & 89.7 & 89.6 & 90.6 \\
			UCE & 2.0 & 2.7 & 2.1 & 2.0 & 1.5 & 2.2 & 2.3 & 2.3 & 1.9 & 3.0 & 2.2 \\
			DP & \textbf{0.5} & \textbf{0.9} & \textbf{0.3} & \textbf{1.4} & \textbf{-0.1} & \textbf{0.4} & \textbf{1.2} &\textbf{ -0.1} & \textbf{1.0} & \textbf{0.3} & \textbf{0.6} \\
			\bottomrule
	\end{tabular}}
	\vspace{-0.15in}
\end{table*}

In terms of \textbf{preservation}, the proposed DP algorithm consistently achieves the smallest degradation across all objects, demonstrating a clear advantage in maintaining non-target representations.
While competing approaches often introduce secondary distortions, such as performance drops exceeding 20\% for ESD and CP, DP preserves nearly unchanged accuracy on the remaining nine categories, typically within only a few percentage points. This stability stems from its double-projection mechanism, particularly the nullspace projection, which explicitly constrains updates to the left nullspace of preserved representations.
Consequently, the erasure operation remains geometrically orthogonal to the non-target embeddings, ensuring that both the visual quality and semantic fidelity of unaffected generations are largely retained.

\subsection{Why is perfect preservation not observed?}
Beyond the observed performance improvements, it is also important to rethink why perfect preservation is not achieved in this experiment. Ideally, non-target concepts should remain entirely unaffected under DP, since the update $\Delta W$ is explicitly designed to be orthogonal to their embeddings, as established in Theorem~\ref{thm:orth-pres}. In practice, however, perfect preservation is not always observed.
This minor deviation arises from the presence of positional embeddings in diffusion models: although DP enforces $\Delta W c_j = 0$ to preserve non-target content embeddings, the model operates on representations of the form $z_j = c_j + q_j$, where $q_j$ denotes the positional embedding.
This additive coupling, which is also present in other closed-form methods such as UCE, introduces small but consistent deviations from perfect preservation, as confirmed empirically (see Appendix~C for details). Moreover, this issue is further compounded by the self-attention mechanism in the encoder, which introduces additional interactions across token representations.

In the following FLUX example, we demonstrate that this fluctuation can be mitigated by performing concept erasure directly on the embedding layers of our encoders.


\subsection{Artistic Styles Erasure with Stable Diffusion}
We next evaluate the proposed DP algorithm on the task of artistic style erasure, using Stable Diffusion 1.4 as the base model. Following prior studies~\cite{gandikota2024unified}, we focus on five representative artistic styles that exhibit diverse visual characteristics and degrees of abstraction. Performance is evaluated using the CLIP text–image similarity between generated images and their corresponding style prompts.

As shown in Table~\ref{tab:sd14_style}, existing methods already demonstrate competitive performance across several artistic styles. For instance, CP achieves particularly strong removal on ``Picasso'', while AGE performs slightly better on ``Caravaggio'', indicating that localized or style-specific optimization can yield strong suppression. Nevertheless, across all styles, DP achieves comparable or better erasure quality on all five artistic styles. The main advantage of DP lies in its strong ability to preserve non-target concepts. While other methods often degrade unrelated styles due to overlapping feature directions, DP ensures that style-independent components remain relatively intact. As a result, the model retains its ability to accurately reproduce unaffected artistic styles with minimal performance drop, typically within only a few percentage points.

\begin{figure*}[t]
	\centering
	\includegraphics[width=\linewidth]{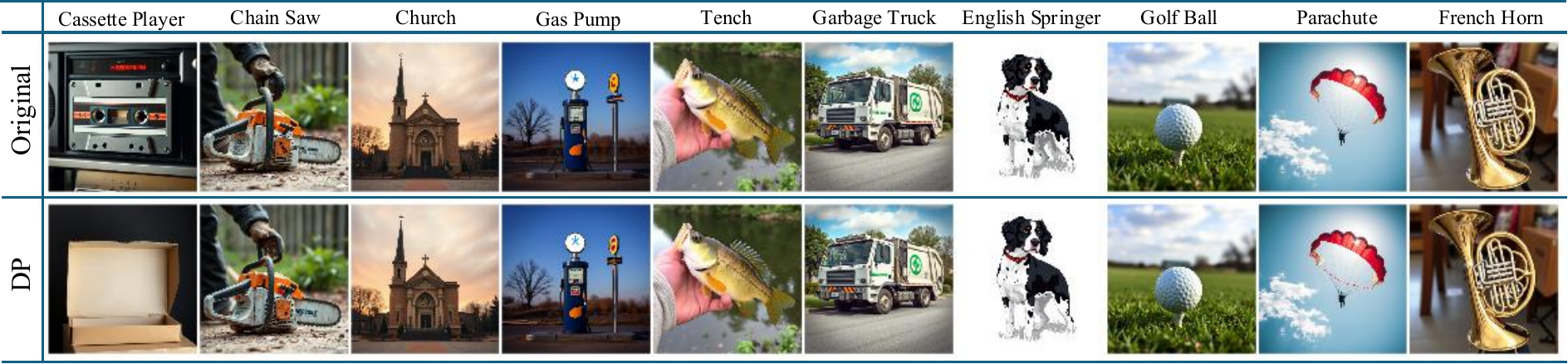}
	\vspace{-0.25in}
	\caption{Visualization of concept erasure on “Cassette player”. Results indicated that the target concept (column 1) is effectively suppressed, while the remaining nine categories (columns 2--10) show minimal impact.}
		\label{fig:fluxcassetteplayer}
	\vspace{-0.1in}
\end{figure*}

\vspace{-0.05in}
\subsection{Switching to Flow Matching}
\vspace{-0.05in}
Recent advances~\cite{lipman2022flow} demonstrate that flow matching offers an equally powerful and more theoretically grounded alternative. To test the generality of our erasure approach, we further evaluate DP on a recent flow-matching model, FLUX~\cite{batifol2025flux}. Note that ESD operates on predicted noise, whereas flow-matching models predict vector fields, making ESD incompatible with these experiments.There is also no direct support to utilize pruning or adversarial methods in flow matching. Consequently, we exclude these approaches from our comparison. Closed-form erasure methods like UCE and DP, by contrast, exhibit \emph{broader applicability} because they directly operate on linear mappings rather than model-specific generative dynamics. Note that in the FLUX model, we apply these closed-form updates to the \emph{embedding layers} rather than attention blocks. This formulation also naturally eliminates interference from positional embeddings, enabling a cleaner concept erasure process.

\begin{figure}[b]
	\centering
	\vspace{-0.2in}
	\includegraphics[width=0.8\linewidth]{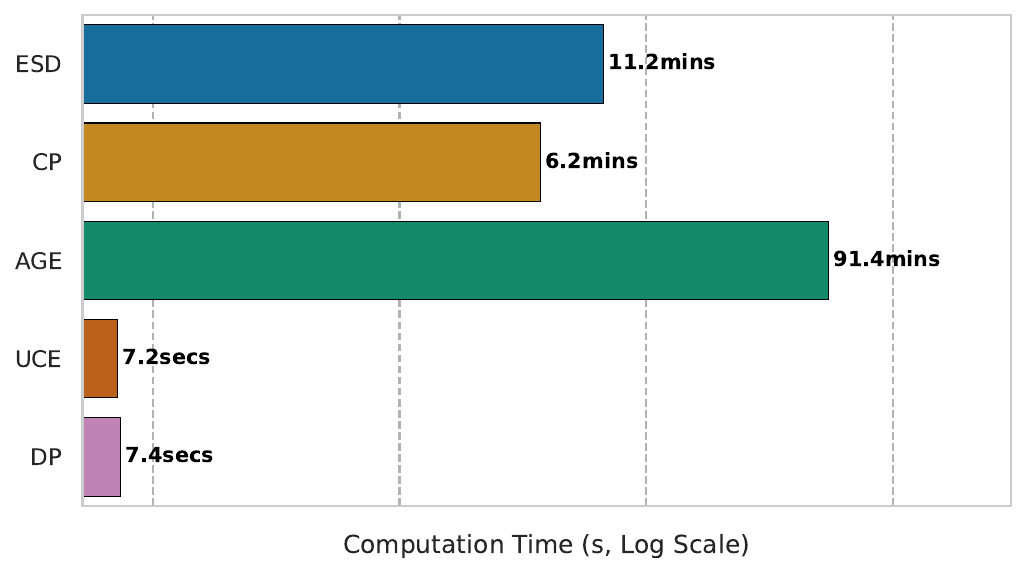}
	\vspace{-0.15in}
	\caption{Computation time comparison across erasure methods (log scale). Results indicate that closed-form approaches such as UCE and DP perform concept erasure within only a few seconds. Experiments are conducted for SD 1.4 on Nvidia 3090.}
	\vspace{-0.2in}
	\label{fig:timecomparison}
\end{figure}

As shown in Table~\ref{tab:flux_obj}, both closed-form methods, UCE and DP, achieve effective object erasure within the flow-matching framework. UCE successfully suppresses most target concepts, demonstrating its adaptability beyond diffusion models; however, its residual accuracies on complex categories such as ``Church'' and ``Gas Pump'' suggest that direct linear projections may not fully capture the flow field’s geometric structure. In contrast, DP consistently attains near-zero residual accuracies across all ten objects, confirming its ability to generalize across generative paradigms. Notably, DP also yields markedly smaller preservation drops, averaging only $0.6\%$ compared to UCE’s $23.9\%$, indicating that the nullspace constraint effectively isolates target directions even in a flow-based representation space. Specifically, Figure~\ref{fig:fluxcassetteplayer} presents the generated images from the original FLUX model and our proposed DP approach for the concept “Cassette Player”. Results indicate that the erased concept is effectively suppressed, while the remaining nine categories are largely preserved.
These results highlight that DP maintains its theoretical advantages, while also providing robust, architecture-agnostic concept erasure with minimal interference to non-target concepts on the FLUX model.

\subsection{Time Consumption}
One key advantage of closed-form erasure methods lies in their exceptional \emph{computational efficiency}. Iterative approaches, such as ESD, CP and AGE, require repeated optimization steps to update noise parameters or perform pruning, leading to substantial time costs on the order of several minutes or even hours, as shown in Figure~\ref{fig:timecomparison}. In contrast, closed-form formulations like UCE and DP complete the erasure process almost instantaneously, requiring only a few seconds. This is due to their underlying training-free mechanism.

\subsection{Additional Experiments}
Due to space constraints, we present additional experiments in the Appendix. (1) We provide quantitative measures including LPIPS~\cite{zhang2018lpips}, PSNR~\cite{bovik2000handbook}, SSIM~\cite{wang2004image}, and FID~\cite{heusel2017gans}, in Appendix~E. (2) Moreover, we conduct experiments on alternative model variants, such as Stable Diffusion v1.5 in Appendix~F. (3) We also present complete visualizations of generated images for both Stable Diffusion and FLUX in Appendix~G and Appendix~H, respectively. (4) Additionally, ablation studies are provided in Appendix~I. (5) Generalization beyond $C_{\text{pres}}$ is reported in Appendix~J.

\section{Conclusion}
In this work, we introduced DP, a closed-form, training-free framework for principled concept erasure in generative models. By formulating the task as a pair of sequential projections, first extracting the safe component of a concept, and then constraining updates within the left nullspace of preserved representations, DP offers a deterministic solution with clear geometric interpretability and analytical guarantees. Comprehensive experiments on both Stable Diffusion and the flow-matching model demonstrate that DP achieves erasure quality comparable to or exceeding existing baselines, while consistently minimizing preservation loss. This broad applicability, coupled with a runtime of only a few seconds, highlights the practicality of DP as a drop-in tool for safe and controllable concept erasure.

\clearpage
\section*{Acknowledgements}
This research is partially supported by the National Research Foundation, Singapore, under the NRF fellowship (project No.NRF-NRFF13-2021-0005).

{
	\small
	\bibliographystyle{ieeenat_fullname}
	\bibliography{main}
}

\appendix
\onecolumn

\section*{Appendix Catalogue}
\addcontentsline{toc}{section}{Appendix Catalogue}
\startcontents[appendices]
\printcontents[appendices]{}{1}{}

\clearpage

\section{Proof for Theorems}
\subsection{Additional Notations}
\label{sec:notations}
\paragraph{Notations}
Let $W_0\in\mathbb{R}^{p\times n}$ be the pretrained linear map, and $W=W_0+\Delta W$ the updated map.
Let $C_{\mathrm{tgt}}=[c_1,\ldots,c_T]\in\mathbb{R}^{n\times T}$ be target (to erase) embeddings, and
$C_{\mathrm{pres}}\in\mathbb{R}^{n\times m}$ the preserved set with ${\rm rank}(C_{\mathrm{pres}})=r$.
Let $C_{\mathrm{pres}}=U_1\Sigma V^\top$ be a thin SVD with left singular vectors
$U_1=[u_1,\ldots,u_r]\in\mathbb{R}^{n\times r}$ ordered by singular values
$\sigma_1\ge\cdots\ge\sigma_r>0$.
Fix $k\in\{0,1,\ldots,r\}$ and define
\[
U_k:=[u_1,\ldots,u_k]\in\mathbb{R}^{n\times k},\quad
U_{\mathrm{tail}}:=[u_{k+1},\ldots,u_r]\in\mathbb{R}^{n\times (r-k)}.
\]
Let $U_{\mathrm{out}}\in\mathbb{R}^{n\times (n-r)}$ be an orthonormal basis of the left nullspace of $C_{\mathrm{pres}}$.
Set the $(n-k)$-dimensional orthogonal complement of $\mathrm{span}(U_k)$ as
\[
U_{2,k}:=\big[\,U_{\mathrm{tail}}\ \ \ U_{\mathrm{out}}\,\big]\in\mathbb{R}^{n\times (n-k)},\qquad
P_k:=U_kU_k^\top,\quad I-P_k=U_{2,k}U_{2,k}^\top.
\]
We parameterize the update by
\[
\ \Delta W \;=\; Z\,U_{2,k}^\top,\qquad Z\in\mathbb{R}^{p\times (n-k)}\  \tag{P}
\]
which enforces \(\Delta W\,U_k=0\) and thus \emph{exactly preserves} the top-$k$ preserve directions.

Let the \emph{safe} proxy be $c_i^\star=\Pi_S c_i:=S(S^\top S)^+S^\top c_i$.
Define
\[
C_\perp^{(k)}:=U_{2,k}^\top C_{\mathrm{tgt}}\in\mathbb{R}^{(n-k)\times T},\qquad
B:=W_0\big(C_{\mathrm{tgt}}^\star-C_{\mathrm{tgt}}\big)\in\mathbb{R}^{p\times T}.
\]
The second projection reduces to a matrix least-squares problem
\[
\min_Z \ \|Z\,C_\perp^{(k)} - B\|_F^2,
\]
whose minimum-norm closed form is
\[
\ Z^\star \;=\; B\,\big(C_\perp^{(k)}\big)^\top\Big(C_\perp^{(k)}\big(C_\perp^{(k)}\big)^\top\Big)^{+},\qquad
\Delta W^\star=Z^\star U_{2,k}^\top.\ \tag{CF}
\]

\subsection{Proof for Theorem~\ref{thm:perturbations}}

\begin{proof}
	From the closed-form solution of UCE, the weight update is given by
	\[
	\Delta W := W_{\mathrm{uce}} - W_0 = (v^* - W_0 c)\, c^\top N^{-1},
	\]
	where \( N = c c^\top + C_{\mathrm{pres}} C_{\mathrm{pres}}^\top \).
	This expression shows that $\Delta W$ is a rank-one update: it modifies the weights in the direction of the residual $(v^* - W_0 c)$, scaled by a transformed version of the target vector $c$ through $N^{-1}$.
	
	To analyze how $\Delta W$ affects different representations, consider its action on the target vector $c$ and a preserve vector $p$:
	\[
	\begin{aligned}
		\Delta c &= \Delta W c = (v^* - W_0 c)\,(c^\top N^{-1} c), \\
		\Delta p &= \Delta W p = (v^* - W_0 c)\,(c^\top N^{-1} p).
	\end{aligned}
	\]
	Both perturbations are proportional to the same direction $(v^* - W_0 c)$, but differ in magnitude depending on how $p$ aligns with $c$ in the metric defined by $N^{-1}$.

	The coefficients $c^\top N^{-1} p$ and $c^\top N^{-1} c$ can be rewritten as inner products in a transformed space:
	\[
	c^\top N^{-1} p = \langle N^{-1/2} c, N^{-1/2} p \rangle, \qquad
	c^\top N^{-1} c = \langle N^{-1/2} c, N^{-1/2} c \rangle.
	\]
	This formulation highlights that the relative effect of $\Delta W$ on $p$ depends on the correlation between $N^{-1/2}c$ and $N^{-1/2}p$.  
	If $p$ is well aligned with $c$ under this transformation, it will inevitably experience a nontrivial perturbation when $c$ is edited.

	Since both $\Delta p$ and $\Delta c$ are parallel to $(v^* - W_0 c)$, their Euclidean norms differ only by the magnitude of the scalar coefficients:
	\[
	\|\Delta p\|_2 = |c^\top N^{-1} p|\,\|v^* - W_0 c\|_2, 
	\quad
	\|\Delta c\|_2 = |c^\top N^{-1} c|\,\|v^* - W_0 c\|_2.
	\]
	Taking their ratio gives
	\[
	\frac{\|\Delta p\|_2}{\|\Delta c\|_2}
	= \frac{|c^\top N^{-1} p|}{|c^\top N^{-1} c|}
	= \frac{|\langle N^{-1/2} c, N^{-1/2} p \rangle|}
	{|\langle N^{-1/2} c, N^{-1/2} c \rangle|}.
	\]

	By the theorem’s assumption, there exists a constant $\lambda > 0$ such that
	\[
	\langle N^{-1/2} c, N^{-1/2} p \rangle 
	\ge \lambda \, \langle N^{-1/2} c, N^{-1/2} c \rangle.
	\]
	Substituting this condition into the previous ratio yields
	\[
	\frac{\|\Delta p\|_2}{\|\Delta c\|_2} \ge \lambda,
	\]
	or equivalently,
	\[
	\|\Delta W p\|_2 = \|\Delta p\|_2 
	\ge \lambda \|\Delta c\|_2 
	= \lambda \|\Delta W c\|_2.
	\]
\end{proof}

\paragraph{Remark:}
This theorem indicates that if there exists a preserved vector $p$ that has a large projection on the target vector $c$ in the $N^{-1/2}$-weighted inner product space, then the perturbation on the preserved vector $p$ may also be comparable to that on the target vector $c$, leading to potential performance degradation on the corresponding preserved concept.

\subsection{Proof for Theorem~\ref{thm:preservation_bound}}
\begin{proof}
	\textbf{Step 1: Exact preservation on the top-$k$ subspace.}
	By construction $U_{2,k}^\top U_k=0$, hence
	\[
	\Delta W\,U_k \;=\; Z\,U_{2,k}^\top U_k \;=\; 0.
	\]
	Therefore for any $p_\parallel\in \mathrm{span}(U_k)$,
	\(
	W\,p_\parallel=(W_0+\Delta W)\,p_\parallel = W_0\,p_\parallel
	\),
	i.e., all top-$k$ principal directions are preserved exactly.
	
	\smallskip
	\textbf{Step 2: The update only acts on the $(I-P_k)$-component.}
	For a general $p\in\mathbb{R}^n$, decompose $p=p_\parallel+p_\perp$ with
	$p_\parallel=P_k p$ and $p_\perp=(I-P_k)p$. Since $\Delta W\,P_k=0$,
	\[
	(W-W_0)p \;=\; \Delta W\,p \;=\; \Delta W\,p_\perp
	\;=\; Z^\star U_{2,k}^\top p_\perp
	\;=\; Z^\star U_{2,k}^\top p.
	\]
	Taking norms and using $\|U_{2,k}^\top p\|_2=\|p_\perp\|_2$ gives
	\[
	\|(W-W_0)p\|_2 \;\le\; \|Z^\star\|_2 \,\|p_\perp\|_2.
	\]
	
	\smallskip
	\textbf{Step 3: Tail bound for preserved columns.}
	Let $p=p_i$ be a column of $C_{\mathrm{pres}}$.
	With the thin SVD $C_{\mathrm{pres}}=U_1\Sigma V^\top$,
	\[
	p_i \;=\; C_{\mathrm{pres}} e_i \;=\; U_1\Sigma V^\top e_i
	\;=\; \sum_{j=1}^r \sigma_j(C_{\mathrm{pres}})\,v_{ij}\,u_j.
	\]
	Since $P_k$ projects onto $\mathrm{span}(u_1,\ldots,u_k)$, the residual is
	\[
	\|(I-P_k)p_i\|_2^2 \;=\; \sum_{j>k} \sigma_j^2(C_{\mathrm{pres}})\,(v_{ij})^2
	\;\le\; \sigma_{k+1}^2(C_{\mathrm{pres}})\sum_{j>k}(v_{ij})^2
	\;\le\; \sigma_{k+1}^2(C_{\mathrm{pres}}).
	\]
	Hence $\|(I-P_k)p_i\|_2\le \sigma_{k+1}(C_{\mathrm{pres}})$, and combining
	with Step 2 yields
	\[
	\|(W-W_0)p_i\|_2 \;\le\; \|Z^\star\|_2\,\sigma_{k+1}(C_{\mathrm{pres}}),
	\]
	as claimed in the theorem statement (with $W'$ replaced by $W$ in our notation).
\end{proof}

\paragraph{Remark:}
Our proposed method can \emph{exactly preserve} the top-$k$ principal directions of the preserved subspace. Moreover, for any preserved column $p_i$, the perturbation norm is upper bounded by the tail energy beyond the top-$k$ singular vectors, scaled by the problem-dependent factor $\|Z^\star\|_2$.
In the case where the rank of $C_{\mathrm{pres}}$ is less than or equal to $k$, the preserved set is \emph{exactly} preserved.

\subsection{Proof for Theorem~\ref{thm:erasure_bound}}
\begin{proof}
	\textbf{First identity.}
	By the closed form \textup{(CF)},
	\[
	Z^\star C_\perp^{(k)}
	= B\,(C_\perp^{(k)})^\top\Big(C_\perp^{(k)}(C_\perp^{(k)})^\top\Big)^{+} C_\perp^{(k)}.
	\]
	Recall the standard pseudoinverse projection identity:
	for any matrix $X$, $X^\top(XX^\top)^{+}X$ is the orthogonal projector
	onto $\mathrm{row}(X)$.
	With the thin SVD
	$C_\perp^{(k)} = U_{q_k}^{(k)}\Sigma_{q_k}^{(k)} V_{q_k}^{(k)\top}$,
	this projector equals $V_{q_k}^{(k)}V_{q_k}^{(k)\top}$.
	Hence
	\[
	Z^\star C_\perp^{(k)} \;=\; B\,V_{q_k}^{(k)}V_{q_k}^{(k)\top}.
	\]
	Using $C_\perp^{(k)}=U_{2,k}^\top C_{\mathrm{tgt}}$,
	\[
	(W-W_0)C_{\mathrm{tgt}}
	= \Delta W^\star C_{\mathrm{tgt}}
	= Z^\star U_{2,k}^\top C_{\mathrm{tgt}}
	= Z^\star C_\perp^{(k)}
	= B\,V_{q_k}^{(k)}V_{q_k}^{(k)\top},
	\]
	which proves the first statement.
	
	\smallskip
	\textbf{Per-column lower bound.}
	Fix a target column index $i$ and set
	$y_i:=V_{q_k}^{(k)\top} e_i$.
	By taking the $i$-th column of the previous identity,
	\[
	(W-W_0)c_i
	= (W-W_0)C_{\mathrm{tgt}} e_i
	= B\,V_{q_k}^{(k)}V_{q_k}^{(k)\top} e_i
	= B\,V_{q_k}^{(k)} y_i.
	\]
	Under the hypothesis that $B V_{q_k}^{(k)}$ has full column rank,
	its smallest singular value $\sigma_{\min}(B V_{q_k}^{(k)})$ is strictly positive, and the standard singular-value inequality yields
	\[
	\|(W-W_0)c_i\|_2
	= \|B V_{q_k}^{(k)} y_i\|_2
	\;\ge\; \sigma_{\min}(B V_{q_k}^{(k)}) \,\|y_i\|_2,
	\]
	which is exactly the claimed bound.
\end{proof}

\paragraph{Remark:}
Our method can exactly fit the part of the target update $B$ that lies in the identifiable row space
$\mathrm{row}(C_\perp^{(k)})$.
Under a mild compatibility assumption, each target column with nonzero leverage in this row space
($\|V_{q_k}^{(k)\top} e_i\|_2>0$) is guaranteed to be modified by at least
$\sigma_{\min}(B V_{q_k}^{(k)})$ times its leverage $\|V_{q_k}^{(k)\top} e_i\|_2$.
In particular, in the single-target case ($T=1$), if the target concept $c$ is not contained in the
top-$k$ preserved subspace (i.e., $U_k^\top c\neq c$), then the pseudo-inverse solution gives an
exact fit on the erased concept, i.e. $Wc=W_0 c^\star$.

\vspace{0.4in}
\section{Detailed Experiment Settings}
\subsection{Object Erasing}
To ensure a fair and controlled comparison across all erasure methods, we assign a fixed anchor concept to each target object category. This guarantees that UCE and DP operate under identical proxy vectors $v_i^*$, thereby isolating differences in performance to the erasure mechanisms themselves rather than to variations in replacement semantics. For every target concept, the chosen anchor represents a semantically neutral or structurally compatible object, enabling a clear evaluation of how effectively each method suppresses the target while redirecting the model toward the specified substitute.

\begin{table}[h]
	\centering
	\caption{Anchor concepts used for object-level concept erasure. Each target is paired with a fixed anchor to ensure consistent proxy vectors across UCE and DP.}
	\vspace{0.1in}
	\resizebox{\linewidth}{!}{
		\begin{tabular}{lcccccccccc}
			\toprule
			\textbf{Target Concept} 
			& Cassette Player & Chain Saw & Church & English Springer & French Horn & Garbage Truck & Gas Pump & Golf Ball & Parachute & Tench \\
			\midrule
			\textbf{Anchor Concept} 
			& Box & Stick & Temple & Cat & Drum & Bus & Dispenser & Sphere & Cloth & Cucumber \\
			\bottomrule
	\end{tabular}}
	\label{tab:anchors-horizontal}
\end{table}

The choice of anchors in Table \ref{tab:anchors-horizontal} follows the suggestions by ChatGPT 4.1, by considering the semantic meanings. These anchor selections are kept consistent across all visual and quantitative evaluations. This standardized setup allows us to directly compare how different methods respond to identical replacement instructions, and it highlights the extent to which each algorithm both eliminates the target concept and preserves the integrity of non-target categories.

Note that in the above settings, we simplify the first projection step by defining the safe subspace using a single anchor vector. An exploration of more complex constructions of the safe region is provided in Appendix~\ref{sec:ablations}.

\subsection{Evaluation Protocol for Object Erasure}
As alluded to earlier, our evaluation procedure is designed to be more stringent and representative of real-world semantic distinctions than those used in prior work. To reduce ambiguity in classifier predictions and diffusion outputs, we merge concept labels that are visually or semantically close. For example, categories such as ``cassette player,'' ``tape player,'' and ``tape'' are treated as a single unified concept. This consolidation avoids overcounting near-duplicate labels and results in a more accurate and conservative estimate of how well a model retains or erases a target concept. Consequently, the baseline Stable Diffusion model exhibits notably higher accuracy under our protocol, reflecting the increased ability of the unified classification task.

In addition, whereas earlier studies~\cite{bui2025age_iclr2025} often rely on Top-5 accuracy, we report exclusively Top-1 accuracy to provide a stricter and more discriminative measure of model behavior. This choice ensures that all reported metrics reflect precise, single-label correctness rather than broader category inclusion.

\vspace{0.1in}
\subsection{Artist Style Erasure}
For each artist, we construct an extended label set to evaluate style erasure more comprehensively. 
Using ChatGPT-generated descriptors, each label set includes: (i) the five primary artists considered in this study, 
(ii) general artistic categories such as ``classical art'' and ``modern art'', and 
(iii) artist-specific descriptive phrases that capture characteristic stylistic elements 
(e.g., ``dramatic lighting'', ``colorful landscape'', ``abstract portrait''). 
These labels allow us to assess both direct stylistic removal and potential drift toward semantically related artistic styles. 
The complete label sets used in our experiments are listed below.

\begin{itemize}
	\item \textbf{Andy Warhol:} 
	Pablo Picasso, Vincent van Gogh, Rembrandt, Andy Warhol, Caravaggio, 
	Campbell's soup can, Marilyn Monroe portrait, screen printing, celebrity portrait, 
	modern art, classical art.
	
	\item \textbf{Caravaggio:}
	Pablo Picasso, Vincent van Gogh, Rembrandt, Andy Warhol, Caravaggio, 
	Calling of Saint Matthew, Judith Beheading Holofernes, tenebrism, dramatic lighting, religious scene, 
	Baroque, Renaissance, 17th-century art, classical art, realism.
	
	\item \textbf{Pablo Picasso:}
	Pablo Picasso, Vincent van Gogh, Rembrandt, Andy Warhol, Caravaggio, 
	Guernica, Blue Period, Rose Period, African mask, guitar collage, 
	abstract portrait, geometric art, modern art, classical art.
	
	\item \textbf{Rembrandt:}
	Pablo Picasso, Vincent van Gogh, Rembrandt, Andy Warhol, Caravaggio, 
	The Night Watch, self-portrait, Saskia portrait, chiaroscuro, 
	Dutch master, Baroque, classical art, impressionism, cubism, modern art.
	
	\item \textbf{Vincent van Gogh:}
	Pablo Picasso, Vincent van Gogh, Rembrandt, Andy Warhol, Caravaggio, 
	generic impressionist painting, abstract expressionism, post-impressionist art, 
	colorful landscape, Starry Night scene, sunflower painting, wheat field artwork, 
	cypress trees, countryside scene, generic modern art, unspecified artist style.
\end{itemize}

\vspace{0.2in}

\section{Imperfection of Preservation}
\label{sec:imperfect_perservation}
\subsection{The Impacts of Positional Embedding}
Although the DP algorithm theoretically enforces orthogonality between erased and preserved subspaces, \emph{perfect preservation} of non-target concepts is not always achieved in practice. 
This discrepancy primarily arises from the \emph{positional embedding structure} in diffusion models, where each token embedding is not used in isolation but is \emph{summed} with its positional encoding before entering the attention and MLP layers. 

Formally, let the raw content embedding for a token be \( c_i \in \mathbb{R}^n \) and its positional embedding be \( q_i \in \mathbb{R}^n \). 
The effective input to the model is then
\begin{equation}
	z_i = e_i + q_i.
\end{equation}
During concept erasure, DP computes an update \( \Delta W \) satisfying the preservation constraint
\begin{equation}
	\Delta W e_i = 0,
\end{equation}
which guarantees that all preserved content embeddings \( C_{\mathrm{pres}} \) remain unaffected in the ideal case. 
However, in the actual model, the transformation is applied to the fused embedding \( z_i \), not to \( c_i \) alone. 

Since \( q_i \) is not fixed (the word can appear at arbitrary location) and generally \emph{not orthogonal} to the erased directions, the effective transformation satisfies
\begin{equation}
	\Delta W  z_i \neq 0.
\end{equation}
This residual term introduces a small coupling between erased and preserved subspaces, leading to the minor performance drop observed empirically. Note the non-target concepts can appear in any position, and in general, it is not feasible to also require $\Delta W q_i =0$ for all $q_i$.

Importantly, this limitation is \emph{not unique to DP}. 
Closed-form projection methods such as UCE are subject to the same positional interaction, since they also operate in the linearized embedding space and do not explicitly disentangle positional components. 
In other words, while both DP and UCE guarantee subspace orthogonality for pure content embeddings, the \emph{additive nature of positional encodings} inherently prevents perfect preservation in diffusion architectures.

\subsection{Self-Attention in Encoder}
Specifically, the CLIP text encoder used in diffusion models applies multiple \emph{self-attention blocks} when producing text embeddings, so the resulting embedding of each token is no longer independent of the others. As a result, token representations become contextualized and partially mixed across the prompt. For example, in the prompt ``An image of Church'', the embedding associated with ``Church'' after encoding is not merely the isolated concept embedding of ``Church'', but a contextualized representation that also carries weak information from the surrounding tokens through self-attention. Consequently, even if $\Delta W$ is constructed to be orthogonal to non-target concept embeddings in principle, the actual encoded representations processed by the model may still be slightly perturbed. Together with the effect of positional embeddings, this token mixing provides a practical explanation for the small but consistent deviations from perfect preservation observed in our experiments.

\vspace{0.1in}
\subsection{Why Editing on Embedding Layers Produces Better Preservation.}
Operating directly on the \emph{embedding layer} of the encoder avoids the positional–intervene concepts before the positional embedding and self-attention. At the embedding layer, the model processes the content vectors $c_i$ \emph{before} they are fused with positional embeddings. This allows the preservation constraint to be enforced exactly.

When concept erasure is applied at the embedding layer, the update $\Delta W_{\mathrm{emb}}$ acts only on $c_i$:
\[
z_i = (W_{\mathrm{emb}} + \Delta W_{\mathrm{emb}})\, c_i + q_i.
\]
The preservation condition becomes
\[
\Delta W_{\mathrm{emb}}\, C_{\mathrm{pres}} = 0,
\]
which directly implies
\[
(W_{\mathrm{emb}} + \Delta W_{\mathrm{emb}})\, C_{\mathrm{pres}}
= W_{\mathrm{emb}}\, C_{\mathrm{pres}}.
\]
Since positional embeddings are added \emph{after} the content projection, they do not interfere with this constraint. The effective representation remains
\[
z_i' = W_{\mathrm{emb}}\, C_{\mathrm{pres}} + q_i,
\]
which is identical to the original representation for all preserved concepts.

These results explain why embedding-level editing consistently yields more stable preservation behavior: it achieves exact orthogonality for content embeddings, results in cleaner and more localized updates, and eliminates interference caused by positional encodings, as further demonstrated in our FLUX visualizations (Appendix~\ref{sec:vis_flux}).

\vspace{0.4in}

\section{Why UCE Underperforms on \textit{Cassette Player} and \textit{Golf Ball}}
\label{sec:tuning}

\subsection{Last Word Erasing}
Although UCE generally provides strong erasure performance, we observe two notable failure
cases in our experiments: \emph{Cassette Player} and \emph{Golf Ball}. Upon closer inspection, 
these failures arise from the way UCE constructs the concept embedding used for editing.

\paragraph{UCE uses only the last token embedding.}
In the official implementation of UCE, the concept embedding for a multi-word prompt is 
constructed by selecting \emph{only the last token} of the prompt. The relevant code snippet 
from the official release is shown below:

\begin{lstlisting}[style=pythonstyle]
	t_emb = pipe.encode_prompt(
		prompt=e,
		device=device,
		num_images_per_prompt=1,
		do_classifier_free_guidance=False)
	
	last_token_idx = (
		pipe.tokenizer(
			e,
			padding="max_length",
			max_length=pipe.tokenizer.model_max_length,
			truncation=True,
			return_tensors="pt",
		)["attention_mask"]
	).sum() - 2
	
	uce_erase_embeds[e] = t_emb[0][:, last_token_idx, :]
\end{lstlisting}

In particular, the above codes effectively select the ``last token index'' for all target concept. For many artistic concepts such as ``Van Gogh'' or ``Picasso'', this design choice is 
relatively benign because the semantic meaning is concentrated in the final token. 
However, for compound nouns commonly found in the object-erasing benchmark, the 
last token does \emph{not} capture the dominant semantics.

\paragraph{Why this fails for ``Cassette Player'' and ``Golf Ball''.}
In both of these categories, the first token carries the primary semantic load: 
``cassette'' in ``cassette player'' and ``golf'' in ``golf ball''.  
UCE, however, replaces only the second token.  
For example:
\begin{itemize}
	\item Replacing \texttt{ball} with \texttt{sphere} leads to prompts interpreted 
	by the model as ``golf sphere'', which often continues to produce 
	golf-ball--like objects. Such outputs remain highly classifiable as 
	\emph{golf ball} by the pretrained ResNet-50 classifier.
	\item Similarly, replacing \texttt{player} in \texttt{cassette player} fails 
	to remove the defining visual features associated with the first token, 
	causing the resulting images to retain the appearance of a cassette-like object.
\end{itemize}

This explains the substantially higher erasing accuracy for UCE on these two categories 
reported in Table~\ref{tab:sd14_objects}.

\paragraph{Replacing all tokens improves UCE in these cases.}
For completeness, we run an additional experiment in which UCE replaces the 
embeddings of \emph{all} tokens in the target phrase rather than only the last one.  
Under this corrected setting:
\begin{itemize}
	\item The erased accuracy for \emph{Golf Ball} improves dramatically, 
	decreasing from $12.0$ to $2.0$, which is comparable to our DP method.
	\item The accuracy drop on non-target concepts is also reduced, 
	improving from $20.8$ to $16.2$.
\end{itemize}

\vspace{0.2in}

\subsection{DP Still Achieves Better Preservation}
Despite these improvements, UCE still introduces substantially larger perturbations to 
non-target concepts. Under the same corrected setting, DP achieves a much lower 
preservation drop of only $3.3$, demonstrating that even with improved token handling, 
UCE's single-projection update remains more disruptive to unrelated concept directions.

This analysis confirms that UCE's underperformance is primarily due to its reliance on 
the last-token embedding, and that our DP method not only avoids this limitation but 
also maintains significantly better preservation of non-target concepts.

\clearpage
\section{Other Metrics for Image Assessment}
\label{sec:other_metrics}
Beyond classification-based accuracy metrics used in the main experiments, we further evaluate the visual quality and perceptual fidelity of generated images of the FLUX model using several widely adopted generative-model metrics: LPIPS~\cite{zhang2018lpips}, PSNR~\cite{bovik2000handbook}, SSIM~\cite{wang2004image}, and FID~\cite{heusel2017gans}.

LPIPS measures perceptual similarity using deep feature distances, providing sensitivity to semantic changes in image content. PSNR and SSIM quantify pixel-level and structural similarity, respectively, enabling assessment of how closely the edited outputs preserve low-level visual attributes. FID evaluates realism at the distribution level by comparing feature statistics of generated images to those of real images. Together, these metrics offer a complementary perspective on the impact of concept erasure, allowing us to assess not only whether the target concept is successfully suppressed, but also how strongly each method affects the overall perceptual quality and statistical properties of non-target generations. 

It is important to emphasize that these metrics are evaluated \textbf{only on the non-target concepts}. 
Measures such as LPIPS~\cite{zhang2018lpips}, PSNR~\cite{bovik2000handbook}, 
SSIM~\cite{wang2004image}, and FID~\cite{heusel2017gans} quantify differences between 
images generated \emph{before and after} concept erasure, and therefore assume that the 
underlying semantic content should remain consistent across the two states. This assumption 
naturally holds for non-target concepts, where the objective is to preserve visual fidelity 
and minimize unintended perturbations.

In contrast, applying these metrics to the \emph{target} concepts would be inappropriate, 
since concept erasure is explicitly designed to alter (and ideally remove) the original content. 
The images before and after erasure are thus expected to differ substantially, rendering such 
reconstruction-based metrics neither meaningful nor interpretable for evaluating erasure quality.

\begin{table}[h]
	\centering
	\small
	\setlength{\tabcolsep}{4pt}
	\caption{Comparison of DP and UCE across preserved concepts using \textbf{LPIPS}, \textbf{PSNR}, \textbf{SSIM}, and \textbf{FID}. Lower is better for LPIPS and FID; higher is better for PSNR and SSIM.}
	\begin{tabular}{l | cc | cc | cc | cc}
		\toprule
		\multirow{2}{*}{Concept} &
		\multicolumn{2}{c|}{\textbf{LPIPS} $\downarrow$} &
		\multicolumn{2}{c|}{\textbf{PSNR} (dB) $\uparrow$} &
		\multicolumn{2}{c|}{\textbf{SSIM} $\uparrow$} &
		\multicolumn{2}{c}{\textbf{FID} $\downarrow$} \\
		& UCE & DP & UCE & DP & UCE & DP & UCE & DP \\
		\midrule
		Cassette Player   & $0.1170_{\pm 0.0287}$ & $\mathbf{0.0506}_{\pm 0.0248}$ & $19.60_{\pm 2.47}$ & $\mathbf{24.61}_{\pm 4.45}$ & $0.8019_{\pm 0.0580}$ & $\mathbf{0.8883}_{\pm 0.0559}$ & 14.29 & $\mathbf{8.57}$ \\
		Chain Saw         & $0.1113_{\pm 0.0262}$ & $\mathbf{0.0362}_{\pm 0.0339}$ & $19.79_{\pm 2.47}$ & $\mathbf{26.63}_{\pm 4.90}$ & $0.8060_{\pm 0.0600}$ & $\mathbf{0.9141}_{\pm 0.0518}$ & 11.87 & $\mathbf{5.44}$ \\
		Church           & $0.1222_{\pm 0.0282}$ & $\mathbf{0.0388}_{\pm 0.0208}$ & $19.33_{\pm 2.41}$ & $\mathbf{25.73}_{\pm 4.24}$ & $0.7964_{\pm 0.0617}$ & $\mathbf{0.9091}_{\pm 0.0461}$ & 14.25 & $\mathbf{6.86}$ \\
		English Springer  & $0.1225_{\pm 0.0397}$ & $\mathbf{0.0836}_{\pm 0.0378}$ & $19.36_{\pm 2.51}$ & $\mathbf{21.68}_{\pm 3.19}$ & $0.7911_{\pm 0.0608}$ & $\mathbf{0.8404}_{\pm 0.0600}$ & 14.74 & $\mathbf{11.28}$ \\
		French Horn       & $0.1211_{\pm 0.0317}$ & $\mathbf{0.0367}_{\pm 0.0202}$ & $19.75_{\pm 2.40}$ & $\mathbf{26.43}_{\pm 4.03}$ & $0.7995_{\pm 0.0646}$ & $\mathbf{0.9110}_{\pm 0.0483}$ & 14.92 & $\mathbf{6.76}$ \\
		Garbage Truck     & $0.1204_{\pm 0.0396}$ & $\mathbf{0.0465}_{\pm 0.0313}$ & $19.81_{\pm 2.50}$ & $\mathbf{25.29}_{\pm 3.73}$ & $0.8040_{\pm 0.0668}$ & $\mathbf{0.8998}_{\pm 0.0616}$ & 14.56 & $\mathbf{7.99}$ \\
		Gas Pump          & $0.1155_{\pm 0.0301}$ & $\mathbf{0.0470}_{\pm 0.0317}$ & $19.69_{\pm 2.36}$ & $\mathbf{25.29}_{\pm 4.67}$ & $0.8009_{\pm 0.0610}$ & $\mathbf{0.8951}_{\pm 0.0599}$ & 13.07 & $\mathbf{7.59}$ \\
		Golf Ball         & $0.1228_{\pm 0.0268}$ & $\mathbf{0.0362}_{\pm 0.0123}$ & $18.98_{\pm 1.63}$ & $\mathbf{25.16}_{\pm 1.96}$ & $0.7885_{\pm 0.0498}$ & $\mathbf{0.9098}_{\pm 0.0254}$ & 15.30 & $\mathbf{6.90}$ \\
		Parachute        & $0.1159_{\pm 0.0337}$ & $\mathbf{0.0682}_{\pm 0.0192}$ & $19.45_{\pm 2.50}$ & $\mathbf{22.34}_{\pm 2.56}$ & $0.7958_{\pm 0.0599}$ & $\mathbf{0.8603}_{\pm 0.0395}$ & 14.54 & $\mathbf{9.93}$ \\
		Tench            & $\mathbf{0.1121}_{\pm 0.0287}$ & $0.1345_{\pm 0.0588}$ & $\mathbf{19.66}_{\pm 2.48}$ & $18.99_{\pm 3.24}$ & $\mathbf{0.8016}_{\pm 0.0633}$ & $0.7769_{\pm 0.0869}$ & $\mathbf{13.61}$ & 14.97 \\
		\midrule
		\textbf{Average} &
		$\mathbf{0.1181}$ & $0.0578$ &
		$19.54$ & $\mathbf{24.22}$ &
		$0.7996$ & $\mathbf{0.8805}$ &
		14.08 & $\mathbf{8.73}$ \\
		\bottomrule
	\end{tabular}
	\label{tab:dp_uce_metrics}
\end{table}

Table~\ref{tab:dp_uce_metrics} reports a comprehensive comparison between DP and UCE
across ten preserved concepts, evaluated using LPIPS, PSNR, SSIM, and FID. For metrics where lower values indicate better
performance (LPIPS and FID), DP consistently outperforms UCE on
nine out of ten concepts. The only exception is the ``Tench'' class, where UCE
achieves a slightly lower LPIPS score. On average, DP achieves a substantially
lower LPIPS score ($0.0578$ vs.\ $0.1181$), indicating a significantly improved
perceptual similarity to the target images.

For distortion-based metrics where higher values indicate better image fidelity
(PSNR and SSIM), DP again demonstrates favorable behavior. DP
achieves higher PSNR and SSIM values on all concepts except ``Tench'', showing a
robust improvement in reconstruction fidelity. Averaged across all concepts, DP
improves PSNR by approximately $+4.7$~dB over UCE ($24.22$ vs.\ $19.54$) and 
achieves a higher SSIM score ($0.8805$ vs.\ $0.7996$), demonstrating consistently
better structural alignment and visual coherence.

In terms of generative quality, DP achieves notably lower FID scores on
nine out of ten concepts, again with the sole exception of ``Tench''. The average
FID of DP ($8.73$) is substantially lower than that of UCE ($14.08$), indicating
that DP produces more realistic and distribution-consistent outputs.

Overall, the results show that \emph{DP outperforms UCE across all four metrics
	and on nearly all individual concepts}. This demonstrates that DP provides 
superior perceptual similarity, lower distortion, higher structural fidelity, 
and more realistic generative quality when preserving concept-specific image 
content.

\clearpage
\section{Additional Experiments on SD 1.5}
\label{sec:sd15}
While the main paper focuses on Stable Diffusion~1.4 due to its widespread use in prior 
concept-erasure research and its role as a canonical benchmark, we also conduct a parallel 
set of experiments on Stable Diffusion~1.5 to assess the robustness and generality of our 
approach. The SD~1.5 backbone differs from SD~1.4 in both training distribution and visual 
appearance characteristics, making it a meaningful testbed for evaluating consistency across 
model variants.

\begin{table*}[thpb]
	\centering
	\caption{Results on SD~1.5 for all algorithms. Each block includes both original and post-update accuracies. 
		Left: \textbf{Target Class} shows erasure performance (\textbf{Erased Accuracy}~$\downarrow$). 
		Right: \textbf{Other Classes} reports the accuracy of preserved concepts (\textbf{Preservation Drop}~$\downarrow$). Lower values indicate stronger erasure and better preservation.}
	\label{tab:sd15_objects}
	\resizebox{0.98\textwidth}{!}{
		\begin{tabular}{l|c|ccccc|c|ccccc}
			\toprule
			\multirow{2}{*}{\textbf{Object}} &
			\multicolumn{1}{c|}{\textbf{Target Class}} &
			\multicolumn{5}{c|}{\textbf{Erased Accuracy (\%)~$\downarrow$}} &
			\multicolumn{1}{c|}{\textbf{Other Classes}} &
			\multicolumn{5}{c}{\textbf{Preservation Drop (\%)~$\downarrow$}} \\
			\cmidrule(lr){2-2}\cmidrule(lr){3-7}\cmidrule(lr){8-8}\cmidrule(lr){9-13}
			& Original & ESD & CP & AGE & UCE & DP & Original & ESD & CP & AGE & UCE & DP \\
			\midrule
			
			Cassette Player 
			& $60.0$ 
			& $12.0_{\pm1.0}$ & $4.0_{\pm0.5}$ & $33.0_{\pm3.7}$ & $41.0$ & $\mathbf{0.0}$
			& $88.8$ 
			& $25.9_{\pm2.1}$ & $27.4_{\pm1.9}$ & $6.3_{\pm0.4}$ & $20.4$ & $\mathbf{4.2}$ \\
			
			Chain Saw 
			& $76.0$ 
			& $1.0_{\pm0.0}$ & $\mathbf{0.0}_{\pm0.0}$ & $2.0_{\pm0.5}$ & $\mathbf{0.0}$ & $\mathbf{0.0}$
			& $87.0$ 
			& $18.1_{\pm1.8}$ & $29.1_{\pm2.2}$ & $1.6_{\pm0.2}$ & $2.3$ & $\mathbf{0.2}$ \\
			
			English Springer 
			& $95.0$ 
			& $1.0_{\pm0.0}$ & $\mathbf{0.0}_{\pm0.0}$ & $\mathbf{0.0}_{\pm0.0}$ & $\mathbf{0.0}$ & $\mathbf{0.0}$
			& $84.9$ 
			& $43.9_{\pm2.4}$ & $32.6_{\pm1.9}$ & $1.9_{\pm0.3}$ & $2.1$ & $\mathbf{-0.2}$ \\
			
			Parachute 
			& $93.0$ 
			& $5.0_{\pm0.7}$ & $5.0_{\pm0.5}$ & $2.0_{\pm0.5}$ & $\mathbf{0.0}$ & $\mathbf{0.0}$
			& $85.1$ 
			& $10.4_{\pm1.1}$ & $37.7_{\pm2.8}$ & $1.3_{\pm0.2}$ & $\mathbf{0.3}$ & $\mathbf{0.3}$ \\
			
			French Horn 
			& $99.0$ 
			& $2.0_{\pm0.3}$ & $\mathbf{0.0}_{\pm0.0}$ & $3.0_{\pm1.0}$ & $\mathbf{0.0}$ & $1.0$
			& $84.4$ 
			& $17.1_{\pm1.6}$ & $29.3_{\pm2.3}$ & $4.7_{\pm0.4}$ & $\mathbf{1.6}$ & $2.4$ \\
			
			Golf Ball 
			& $100.0$
			& $17.0_{\pm1.5}$ & $12.0_{\pm1.0}$ & $2.0_{\pm0.3}$ & $61.0$ & $\mathbf{0.0}$
			& $84.3$
			& $16.2_{\pm1.3}$ & $27.0_{\pm2.0}$ & $4.9_{\pm0.5}$ & $7.6$ & $\mathbf{4.8}$ \\
			
			Garbage Truck 
			& $93.0$
			& $1.0_{\pm0.3}$ & $\mathbf{0.0}_{\pm0.0}$ & $11.0_{\pm1.7}$ & $\mathbf{0.0}$ & $\mathbf{0.0}$
			& $85.1$
			& $18.0_{\pm1.5}$ & $39.2_{\pm3.1}$ & $1.1_{\pm0.2}$ & $-1.7$ & $\mathbf{-1.9}$ \\
			
			Tench 
			& $80.0$
			& $\mathbf{0.0}_{\pm0.0}$ & $\mathbf{0.0}_{\pm0.0}$ & $8.0_{\pm1.7}$ & $\mathbf{0.0}$ & $\mathbf{0.0}$
			& $86.6$
			& $23.6_{\pm2.2}$ & $25.4_{\pm1.7}$ & $2.9_{\pm0.3}$ & $6.0$ & $\mathbf{2.2}$ \\
			
			Gas Pump 
			& $76.0$ 
			& $10.0_{\pm1.0}$ & $2.0_{\pm1.0}$ & $\mathbf{1.3}_{\pm0.3}$ & $3.0$ & $4.0$
			& $87.0$
			& $13.6_{\pm1.3}$ & $34.8_{\pm2.8}$ & $1.6_{\pm0.2}$ & $4.6$ & $\mathbf{1.4}$ \\
			
			Church 
			& $88.0$
			& $25.0_{\pm2.0}$ & $\mathbf{0.0}_{\pm0.0}$ & $13.0_{\pm1.0}$ & $\mathbf{0.0}$ & $2.0$
			& $85.8$
			& $27.0_{\pm2.4}$ & $35.8_{\pm3.3}$ & $4.6_{\pm0.4}$ & $8.1$ & $\mathbf{4.3}$ \\
			
			\midrule
			\textbf{Mean} 
			& $86.0$ 
			& $7.4$ & $2.3$ & $7.5$ & $10.5$ & $\mathbf{0.7}$
			& $85.9$ 
			& $21.4$ & $31.8$ & $3.1$ & $5.1$ & $\mathbf{1.8}$ \\
			\bottomrule
		\end{tabular}
		}
\end{table*}

The results on Stable Diffusion~1.5 in Table~\ref{tab:sd15_objects} exhibit trends consistent with those observed for SD~1.4, further confirming that DP generalizes effectively across different diffusion backbones. Across all ten evaluated object categories, DP achieves the \emph{lowest mean erased accuracy (0.7\%)}, outperforming all competing baselines by a substantial margin. In many cases, including \emph{Cassette Player}, \emph{Chain Saw}, \emph{English Springer}, \emph{Parachute}, \emph{Golf Ball}, \emph{Garbage Truck}, and \emph{Tench}, DP completely suppresses the target object, achieving a residual accuracy of 0.0\%. Even in more challenging categories such as \emph{Church} and \emph{Gas Pump}, DP remains competitive, demonstrating that the double-projection mechanism continues to yield effective erasure despite architectural differences between SD~1.4 and SD~1.5.

In terms of preserving non-target concepts, DP again provides the strongest performance. Iterative or pruning-based approaches such as ESD and CP introduce substantial collateral degradation, often exceeding a large preservation drop. UCE performs better but still yields an average drop of 5.1\%. In contrast, DP maintains an average degradation of only 1.8\%, several times lower than any other method. In multiple categories, including \emph{English Springer}, \emph{Garbage Truck}, and \emph{Tench}, DP results in slightly \emph{negative} preservation drop, indicating that the overall accuracy for other non-target concepts increases.

Overall, two clear behavioral clusters emerge. Methods like CP and ESD display high variance and significant unintended perturbations due to their reliance on broad, iterative parameter modifications. UCE performs reasonably on simpler single-token concepts but struggles with multi-token cases (e.g., \emph{Golf Ball}, \emph{Cassette Player}), reflecting the token-selection limitations discussed previously. By contrast, DP remains uniformly stable: its closed-form update isolates the erasure direction while explicitly preserving the orthogonal subspace, enabling it to maintain high fidelity even when concept representations are semantically entangled.

These findings reinforce the central message of this work: DP provides strong, architecture-agnostic concept erasure while consistently minimizing unintended degradation, validating the robustness of the proposed double-projection framework across both classical and updated diffusion model variants.

\clearpage
\section{Visualization for Stable Diffusion 1.4}
\label{sec:vis_sd}
\begin{figure}[h]
	\centering
	\includegraphics[width=1\linewidth]{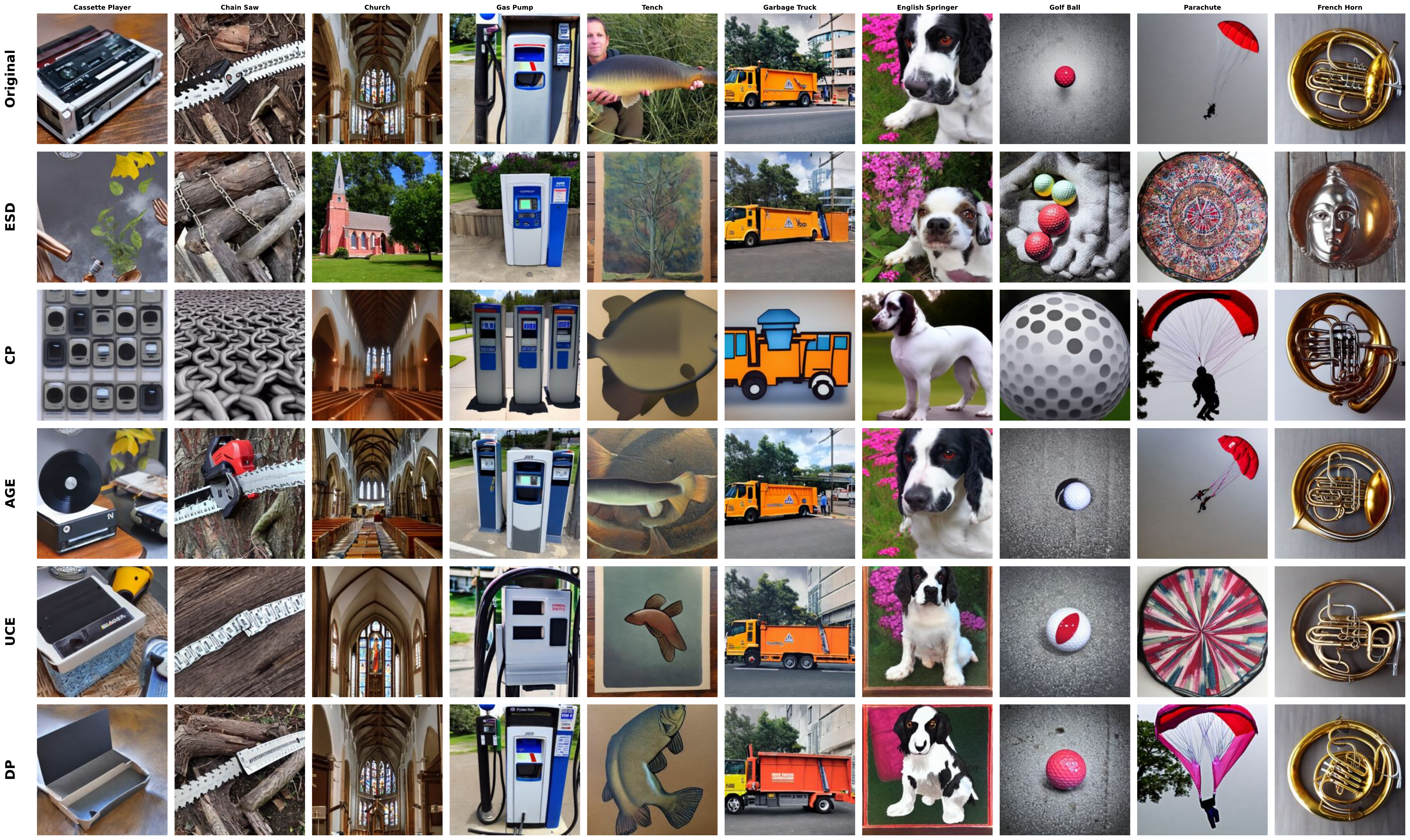}
	\caption{Concept erasure on “Cassette Player” with anchor concept “Box”. The first column shows the target concept to be erased.}
	\label{fig:sd14comparecassetteplayer}
\end{figure}

Figure~\ref{fig:sd14comparecassetteplayer} presents qualitative visualizations for all five erasure methods, ESD, CP, AGE, UCE, and the proposed DP, using the cassette player category as the target concept. For each method, we display the first generated sample from the ten evaluated categories, with the objective of suppressing the target concept in the first column while leaving the remaining nine concepts unaffected.

For the target concept, the closed-form approaches UCE and DP both succeed in preserving the overall structural layout of the original image while substituting the target semantics with the designated anchor concept. Notably, DP produces outputs that align more faithfully with the anchor concept box, yielding clearer and more coherent substitutions than those produced by UCE. This behavior visually corroborates the quantitative results reported earlier, where DP demonstrated stronger erasure performance on the target class.

For the nine non-target concepts, however, all methods exhibit some degree of perturbation. These deviations are especially pronounced for CP, whose outputs diverge substantially from the original images, indicating weaker preservation capability. DP also shows mild perturbations on non-target categories, though the changes are considerably smaller and do not alter the primary semantics of the generated content. 

In contrast, when concept erasure is applied directly to the embedding layer, as demonstrated in FLUX (see Appendix~\ref{sec:vis_flux}), the model preserves non-target concepts much more reliably. This comparison suggests that interventions performed within deeper architectural components, such as attention blocks, are more likely to propagate unintended changes throughout the network. Even with closed-form constraints, edits at these deeper layers can influence representations beyond the targeted concept.

\section{Visualization for FLUX}
\label{sec:vis_flux}
\begin{figure}[htb]
	\centering
	\begin{subfigure}[b]{\textwidth} 
		\centering
		\includegraphics[width=\linewidth]{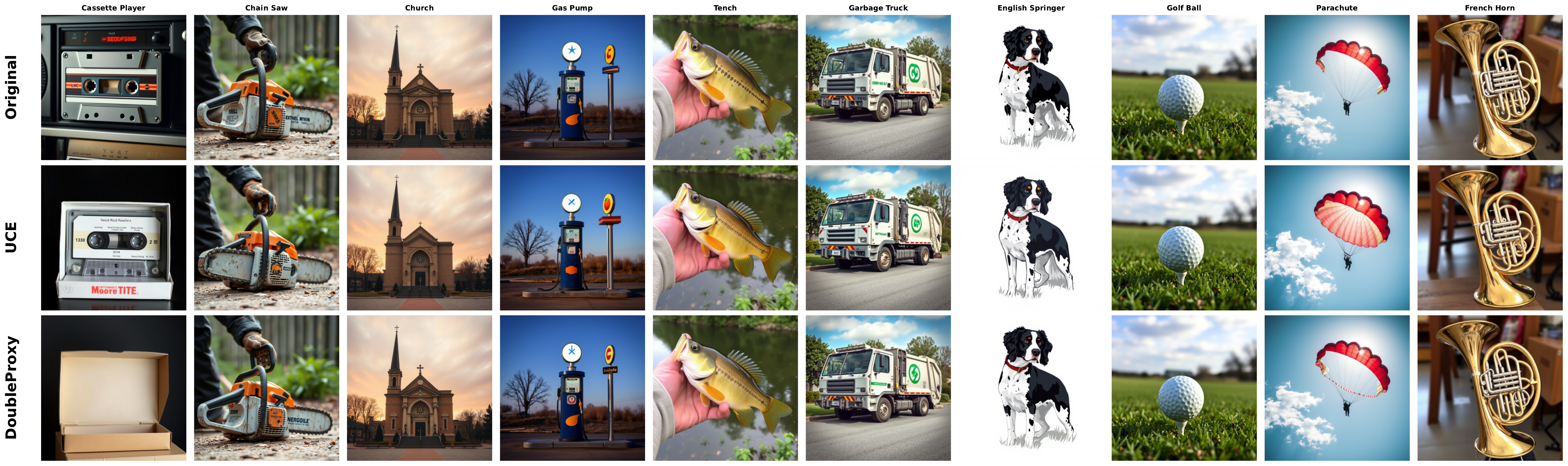}
		\caption{Concept erasure on “Cassette Player” with anchor concept “Box”. The first column shows the target concept to be erased.}
		\label{fig:sub1}
	\end{subfigure}
	\hfill 
	\begin{subfigure}[b]{\textwidth}
		\centering
		\includegraphics[width=\linewidth]{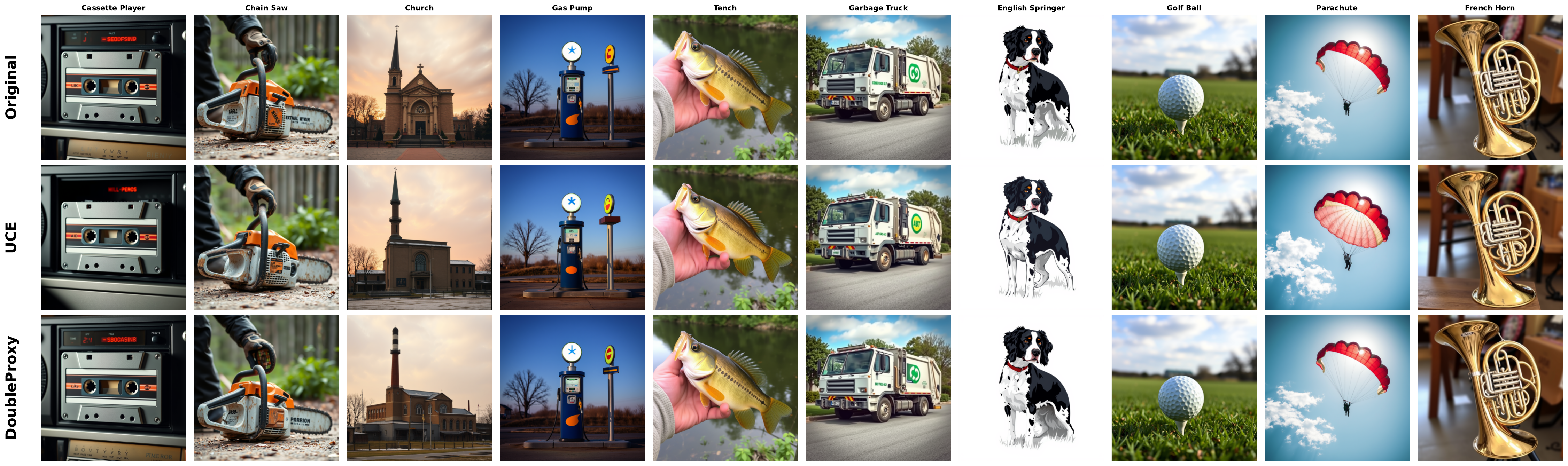}
		\caption{Concept erasure on “Church” with anchor concept "Factory". The third column shows the target concept to be erased.}
		\label{fig:sub2}
	\end{subfigure}
	\hfill
	\begin{subfigure}[b]{\textwidth}
		\centering
		\includegraphics[width=\linewidth]{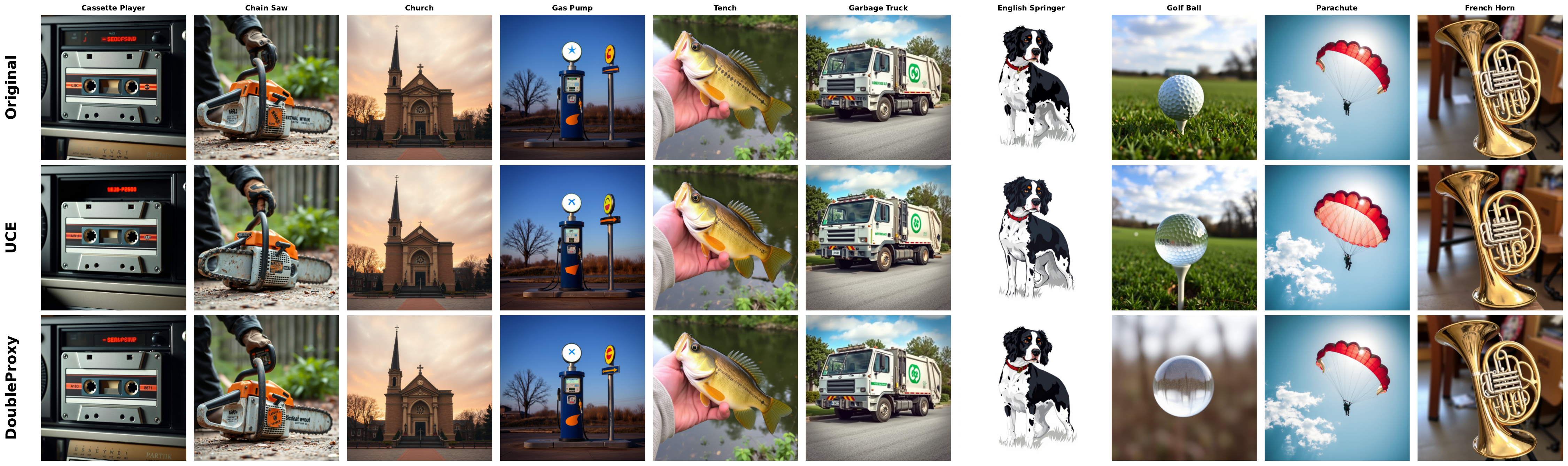}
		\caption{Concept erasure on “Golf Ball” with anchor concept "Sphere". The eighth column shows the target concept to be erased.}
		\label{fig:sub3}
	\end{subfigure}
	\caption{Concept erasure on a few target concepts with FLUX. Results demonstrate that the proposed DP method successfully suppresses the target concept while generating images faithful to the replaced anchor concept.} 
	\label{fig:flux_three_images}
\end{figure}

To further illustrate the qualitative behavior of concept erasure, Figure~\ref{fig:flux_three_images} presents visualizations for three representative target concepts—cassette player, church, and golf ball. For fairness and controlled comparison, each target concept is paired with a fixed anchor concept used as the replacement proxy $v_i^*$. Across all examples, the proposed DP method consistently removes the target concept while producing images that align closely with the intended anchor semantics. In contrast, UCE often retains recognizable traces of the original concept, indicating incomplete suppression.

This difference is most evident in Figure~\ref{fig:sub2}: although both methods attempt to erase the concept church using factory as the anchor, DP produces structures that clearly resemble industrial buildings, whereas UCE-generated images continue to exhibit architectural features characteristic of churches. Similar patterns appear across the remaining examples—DP reliably redirects the model’s output toward the anchor concept, while UCE frequently preserves residual cues associated with the target. These qualitative results reinforce our quantitative findings, demonstrating that DP achieves more effective concept removal and cleaner semantic substitution, thereby validating its superior erasure capability.

\section{Ablation Studies}
\label{sec:ablations}

\subsection{Ablation on the First Projection}
The first component of the DP framework is the first projection step, where the target concept is mapped into a user-defined safe subspace. In the main paper, we adopt a simplified configuration in which the safe subspace is defined by a single anchor concept, mirroring the setup used in UCE to ensure a fair comparison. However, the DP formulation naturally supports larger and more expressive safe regions, which may improve preservation fidelity or alter erasure behavior depending on the geometry of the selected subspace.

To illustrate this effect, we perform an ablation study on the target concept “Church”. Specifically, we compare two settings:
\begin{enumerate}
	\item A multi-vector safe region constructed from the concepts ``tower'' and ``factory''.
	\item A single-vector safe region using only “factory” as the anchor (as in the main experiments).
\end{enumerate}

\begin{table}[h]
	\centering
	\caption{Ablation study on the construction of the safe subspace for the target concept \emph{Church}. 
		We compare a multi-vector safe region (Tower + Factory) with a single-vector anchor (Factory). 
		Left block reports erasure performance; right block reports preservation quality on non-target classes.}
	\label{tab:ablation_safe_region}
	\resizebox{0.7\linewidth}{!}{
		\begin{tabular}{l|ccc|ccc}
			\toprule
			\multirow{2}{*}{\textbf{Method}} &
			\multicolumn{3}{c|}{\textbf{Target Concept: Church}} &
			\multicolumn{3}{c}{\textbf{Other Classes (Avg.)}} \\
			\cmidrule(lr){2-4} \cmidrule(lr){5-7}
			& \textbf{Original} & \textbf{After Erasure}  $\downarrow$ & \textbf{Drop} $\uparrow$
			& \textbf{Original} & \textbf{After Erasure} $\uparrow$ & \textbf{Drop} $\downarrow$ \\
			\midrule
			\textbf{Factory only} &
			99.0\% &  12.0\% & 87.0\% & 
			89.7\% & 89.4\% & 0.3\% \\
			
			\textbf{Tower + Factory} &
			99.0\% & 3.0\% & 96.0\% & 
			89.7\% & 89.6\% & 0.1\% \\
			
			\bottomrule
		\end{tabular}
	}
\end{table}

Table~\ref{tab:ablation_safe_region} demonstrates that broader safe subspaces yield more effective erasure while preserving non-target concepts with minimal degradation. Even so, using a single anchor vector often remains the preferred strategy in practice due to its simplicity and ease of deployment.

\subsection{Ablation on the Second Projection}
The role of the second projection can be directly assessed by comparing our method with UCE (e.g, 
Tables~\ref{tab:sd14_objects}). 
Since both approaches use the \emph{same} anchor vectors $v_i^*$ and differ only in the presence of the 
nullspace projection, these results naturally serve as ablation studies isolating the contribution of the 
second projection. This performance gap highlights the necessity of the second projection: without restricting updates to the 
left nullspace of preserved embeddings, as in UCE, concept removal introduces noticeable interference to 
unrelated representations. In contrast, enforcing the nullspace constraint ensures that modifications remain 
geometrically orthogonal to the preserved subspace, resulting in significantly more stable and predictable 
behavior across both diffusion and flow-matching architectures.

\section{Generation on Other Classes}
\label{sec:generations}
To more comprehensively evaluate the generality of our concept-erasure framework, we conduct an additional set of experiments on a broader collection of ImageNet classes beyond these ten categories used in the main paper. In particular, we focus on the FLUX model for UCE and DP methods. These experiments serve two primary purposes. First, they allow us to examine the stability of our method when applied across a wider range of visual concepts with diverse semantics and visual structures. Second, they enable a deeper analysis of how preservation quality behaves when the preserved concept matrix \(C_{\text{pres}}\) contains classes that differ in similarity to the target concept.

To construct this extended benchmark, we curated a set of seven ImageNet-confirmed synsets spanning multiple semantic domains, including household objects, animals, vehicles, furniture, and sports equipment. The selected classes are:
\texttt{coffee\_mug}, \texttt{beer\_bottle}, \texttt{African\_elephant}, \texttt{airliner}, \texttt{mountain\_bike}, \texttt{loudspeaker}, and \texttt{volleyball}. This selection follows the suggestions from ChatGPT and ensures broad coverage across the ImageNet hierarchy while avoiding redundancy among preserved concepts.

Notably, we intentionally include \texttt{loudspeaker}, which is semantically related to ``Cassette Player''. By doing so, we create a more challenging scenario for evaluating the behavior of \(C_{\text{pres}}\): the preservation matrix now contains a near-neighbor of the erased concept, allowing us to test whether the erasure update can suppress only the target direction without unintentionally diminishing representations associated with semantically adjacent classes. The remaining concepts, chosen to be visually and semantically distinct from the target, provide a stable set for assessing preservation fidelity.

\begin{table}[t]
	\centering
	\small
	\setlength{\tabcolsep}{4pt}
	
	\begin{tabular}{l | c | cc | cc}
		\toprule
		\multirow{2}{*}{Class} &
		\multirow{2}{*}{\textbf{Original}} &
		\multicolumn{2}{c|}{\textbf{UCE}} &
		\multicolumn{2}{c}{\textbf{DP}} \\
		& & Preserved $\uparrow$ & Drop $\uparrow$ & Preserved $\uparrow$ & Drop $\uparrow$ \\
		\midrule
		African elephant & 85.0\% & 79.0\% & +6.0\%          & 84.0\% & \textbf{+1.0\%} \\
		airliner         & 96.0\% & 94.0\% & +2.0\%          & 96.0\% & \textbf{+0.0\%} \\
		beer bottle      & 96.0\% & 91.0\% & +5.0\%          & 96.0\% & \textbf{+0.0\%} \\
		coffee mug       & 71.0\% & 66.0\% & +5.0\%          & 70.0\% & \textbf{+1.0\%} \\
		loudspeaker      & 93.0\% & 95.0\% & -2.0\%          & 96.0\% & \textbf{-3.0\%} \\
		mountain bike    & 100.0\% & 98.0\% & +2.0\%         & 100.0\% & \textbf{+0.0\%} \\
		volleyball       & 39.0\% & 26.0\% & +13.0\%         & 35.0\% & \textbf{+4.0\%} \\
		\midrule
		\textbf{Mean}    & 82.86\% & 78.43\% & 4.43\%          & 82.43\% & \textbf{0.43\%} \\
		\bottomrule
	\end{tabular}
	
	\caption{Classification accuracy comparison on general ImageNet classes before and after concept erasure on "Cassette Player". ``Original'' denotes accuracy on the unmodified model. ``Preserved'' is accuracy after applying UCE or DP. ``Drop'' is defined as (Original $-$ Preserved), where smaller drops (bold) indicate better preservation of general concepts.}
	\label{tab:general_preservation_dp_uce}
\end{table}

To assess whether concept erasure affects recognition performance on unrelated
classes, we evaluate the classification accuracy on seven general ImageNet
categories (Table~\ref{tab:general_preservation_dp_uce}). Since the original
model predictions are identical for both methods, we report them only once and
compare the post-erasure accuracy (``Preserved'') as well as the accuracy drop
(Original $-$ Preserved). A smaller drop indicates better retention of general
concepts unrelated to the targeted erased concepts.

Across the seven categories, DP consistently exhibits smaller drops in accuracy,
achieving an average drop of only 0.43\%, compared to 4.43\%
for UCE. DP matches or outperforms UCE on every class, including
``loudspeaker'' where the drop is negative, indicating an unexpected boost in
accuracy after applying the method. In contrast, UCE frequently induces
substantial degradation, most notably on the ``volleyball'' class where the
accuracy falls by 13 percentage points.

The preserved accuracies further support this trend: DP retains an average of
82.43\% classification accuracy post-erasure, nearly identical to the
original value of 82.86\%. UCE, however, drops to an average of 78.43\%,
showing that the method introduces notable unintended interference in general
recognition capabilities.

Since the target concept in our experiments is ``cassette player''~(the target concept is chosen alphabetically.), it is
natural to examine how erasure interacts with semantically related categories.
Among the evaluated classes, ``loudspeaker'' is arguably the closest in terms of
object type and visual context: both involve audio equipment, share similar
geometric structures, and frequently co-occur in similar environments. One might
reasonably expect such conceptual proximity to induce a noticeable decline in
recognition performance after erasure. 

However, the empirical results reveal that the influence on ``loudspeaker'' is
remarkably minor for both methods. DP exhibits only a $-3.0$\% drop,
while UCE shows a slightly smaller $-2.0$\% drop. Importantly, both
drops are negative, indicating that recognition accuracy actually \emph{improves}
after concept removal. This suggests that the removed ``cassette player''
features are sufficiently specialized and do not interfere with the broader
representation needed to recognize a ``loudspeaker''. The fact that DP maintains
robust performance on this semantically adjacent class, while still achieving
the intended erasure, highlights its ability to localize the targeted concept
without degrading conceptually overlapping regions of the feature space.

Overall, these results demonstrate that DP generalizes more safely: it removes
the targeted concept while preserving recognition performance on unrelated
classes, whereas UCE exhibits measurable collateral damage across diverse
ImageNet categories.

\end{document}